\documentclass[acmsmall]{acmart}

\AtBeginDocument{%
  \providecommand\BibTeX{{%
    \normalfont B\kern-0.5em{\scshape i\kern-0.25em b}\kern-0.8em\TeX}}}

\usepackage{algorithm,algorithmic,subfigure}
\usepackage{todonotes}

\newtheorem{assumption}{Assumption}

\def\R{{\mathbb{R}}}
\def\Pb{{\mathbb P}}
\def\E{{\mathbb E}}
\def\B{{\mathbb B}}

\newcommand{\la}{\langle}
\newcommand{\Fhat}{\widehat{F}}
\newcommand{\ra}{\rangle}
\newcommand{\cT}{\mathcal{T}}
\newcommand{\cM}{\mathcal{M}}
\newcommand{\cB}{\mathcal{B}}
\newcommand{\cF}{\mathcal{F}}
\newcommand{\cJ}{\mathcal{J}}
\newcommand{\PP}{\mathbb{P}}
\newcommand{\cI}{\mathcal{I}}
\newcommand{\cN}{\mathcal{N}}
\newcommand{\cZ}{\mathcal{Z}}

\newcommand{\pihat}{\hat{\pi}}

\begin{document}

\title{Can We Do Better Than Random Start? The Power of Data Outsourcing}

\author{Yi Chen}
\email{yichen@ust.hk}
\affiliation{%
  \institution{Hong Kong University of Science and Technology}
  \city{Hong Kong}
  \country{China}
}

\author{Jing Dong}
\email{jing.dong@gsb.columbia.edu}
\affiliation{%
  \institution{Columbia University}
  \city{New York}
  \state{NY}
  \country{USA}}

\author{Xin T. Tong}
\email{mattxin@nus.edu.sg}
\affiliation{%
  \institution{National University of Singapore}
  \country{Singapore}
}

\renewcommand{\shortauthors}{Chen, Dong, and Tong}

\begin{abstract}
Many organizations have access to abundant data but lack the computational power to process the data.
While they can outsource the computational task to other facilities, 
there are various constraints on the amount of data that can be shared.
It is natural to ask what can data outsourcing accomplish under such constraints.
We address this question from a machine learning perspective.
When training a model with optimization algorithms, 
the quality of the results often relies heavily on the points where the algorithms are initialized. 
Random start is one of the most popular methods to tackle this issue, 
but it can be computationally expensive and not feasible for organizations lacking computing resources.
Based on three different scenarios, we propose simulation-based algorithms that can utilize a small amount of outsourced data to find good initial points accordingly.
Under suitable regularity conditions, we provide theoretical guarantees showing the algorithms can find good initial points with high probability.
We also conduct numerical experiments to demonstrate that our algorithms perform significantly better than the random start approach.  
\end{abstract}

%

\keywords{Non-convex optimization, initialization, }

\maketitle

\section{Introduction}
In this era, data is the new gold. 
Organizations of different sizes and sectors all realize the value of collecting data.
However, it often requires substantial computational power to turn these data into valuable predictive models and
not all organizations have such computational resources.
One possible solution to this problem is outsourcing the data processing task to another computing facility, where computational power is substantially cheaper.
However,  the data organization may only be willing to share a small part of their data due to the following reasons:
First, if the computing facility has access to all the available data, it can obtain an accurate predictive model which leads to potential competition risk.  Second, some parts of the data may not be share-able due to privacy concerns. 
Third, transferring data can be expensive especially when certain encryption is required. 

Given the constraint that only part of the data is ``share-able", 
the organization with data can only expect sub-optimal results from the computing facility, 
and additional learning are needed to improve these premature results.
{Since the data organization is assumed to have limited computational power, it is desirable if the computational cost of the additional learning can be minimized.}
In this context, we are interested in investigating the following two questions:
1) What type of computational task should be assigned to the computing facility?
2) How much data should be outsourced?
{ In this paper, we address these two questions from the perspective of machine learning.}

Most machine learning models are trained using the risk minimization approach. That is, the unknown parameter $\theta$ is inferred by minimizing a loss function of the form 
$F(\theta)=\E[f(\theta,X)]$
where $X$ is averaged over a population distribution or empirical distribution of $N$ data points, and $f(\theta,x)$ is the loss of using the model with parameter $\theta$ to explain the data point $x$. Greedy local optimization algorithms are often applied to minimize $F$. If $F$ is strongly convex, the computational cost of an algorithm $\cT$, $c(\cT)$, depends on the accuracy requirement $\epsilon$ and/or the number of data points $N$. 
In this setting, $c(\cT)$ can be large but is computationally manageable since $\cT$ converges to the optimal parameter regardless of the initialization \cite{johnson2013accelerating, schmidt2017minimizing}.
However, if $F$ is non-convex, the quality of the parameter learned from $\cT$ can depend heavily on its initialization $\theta^0$. In general, greedy algorithms converge to local minimums that are close to $\theta^0$. Thus, in order to find the global minimum $\theta^*$, one needs to start $\cT$ in an appropriate attraction region of optimal parameter $\theta^*$, $\B_0^*$. In practice, the location and shape of $\B_0^*$ is unknown. A common way to deal with this issue is using ``randomized initialization" where the initial points are sampled uniformly at random from the solution space. The idea is that by trying multiple, say $m$, random initializations, one of the initial points will be in $\B_0^*$ and $\cT$ applied to that point will find $\theta^*$. Hence, the total computational cost in this case is $m \cdot c(\cT)$. 

From the above discussion, we note that when learning a non-convex loss function, the computational cost is the product of  two tasks: 1) Exploitation: running greedy algorithm starting from a given initial point and 2) Exploration: finding an initialization within the attraction region of the global minimum. To achieve high accuracy, the exploitation task given a good starting point often requires a sufficiently large amount of data and is very well understood in the literature \citep{bottou2018optimization}.
In contrast, the exploration task is less studied. The performance can be problem dependent and the computational cost can be very high.
One important insight that we will leverage in our subsequent development is that the landscape of the empirical risk based on a random sample of size $n$,
$\Fhat_n(\theta)=\frac1n\sum_{i=1}^n f(\theta,x_i)$,
should resemble that of $F(\theta)$ reasonably well when $n$ is large enough.
Thus, it is natural to ask that, in the data outsourcing context, if we can assign the exploration task to the computing facility. 
In other words, we split the computation tasks into two phases:
\begin{itemize}
\item \noindent{\em Exploration:} The computing facility is assigned to explore the energy landscape of $\Fhat_n$, where $n$ is much smaller than the size of full dataset, and find a good initial point(s) $\theta^0$ (or $\theta_1,\ldots, \theta_L$).
\item \noindent{\em Exploitation:} The data organization can run more refined exploitation starting from $\theta^0$ (or $\theta_1,\ldots, \theta_L$).
In this case, the computational cost, from the data organization's perspective, can be reduced from $m\cdot c(\cT)$ to $c(\cT)$ (or $L \cdot c(\cT)$), where $m$ is the number of random initializations. Such a reduction can be substantial if $m$ needs to be a large number to achieved a desired performance. 
\end{itemize}

Similar computational strategy can also be applied even outside the data outsourcing context. The idea is that we can first use a less accurate loss function $\Fhat_n$ with a smaller amount of data to find good initializations. We then employ greedy optimization algorithms on $F$ starting from these carefully selected initial points.

%

{\bf Our contribution.} 
First,  we propose sampling-based algorithms to obtain good initializations for $F(\theta)$-optimization with outsourced data. Particularly, we design two types of procedures, sampling or optimization, depending on whether the optimization cost $c(\cT)$ is moderate or large:
\noindent{\em If $c(\cT)$ is moderate,} multiple instances of $\cT$ can be implemented starting from different initial points. In this scenario, we suggest using samples from a distribution $\pi_\beta\propto \exp(-\beta \Fhat_n)$ with a properly chosen $\beta$ as initial points.
\noindent{\em If $c(\cT)$ is large,} only one  instance of $\cT$ can be implemented.  In this scenario, we suggest starting from the global minimum of $\Fhat_n$. This minimizer can be obtained by implementing a proper selection procedure on samples from $\pi_\beta$. 

Second, our analytical results provide rigorous justification of these procedures and guide how much data should be outsourced. In particular, we show that under appropriate conditions, when $n\geq O(d\log (1/\rho)/\delta^2)$, with probability $(1-\rho)$, both methods can find a good initial point. Here, $d$ is the dimension parameter and $\delta$ is a parameter for the approximation accuracy, which may depend on the structure of the objective function. Under proper regularity conditions, when $\cT$ is initialized from the point(s) output from the exploration stage, with a high probability, it will find the global minimizer of $F$.

Noticeably, our procedures are compatible with the data outsourcing setup. In particular, the computing facility has access to  only $n$ data points. It will carry out either the sampling or the optimization procedure on $\Fhat_n$ to generate good initial point(s). The data organization can then run a greedy optimization algorithm starting from these point(s) to optimize $F$. The data organization saves in-house 
computational effort in the second optimization stage. Meanwhile, it only exposes  $n$ data points to outside parties.  


{\bf Related literature.}
Data outsourcing has been a problem of intensive interest in the last decade due the emergence of big data and cloud computing.
Most existing works focus on data management policies and encryption \citep{di2007data, foresti2010preserving, samarati2010data}. To the best of our knowledge, this work is the first to study data outsourcing from a machine learning perspective. 

Our problem can be viewed as a special non-convex stochastic
optimization problem. How to efficiently solve non-convex stochastic programing is a fast developing area \cite{ghadimi2016accelerated, wang2017stochastic, allen2017natasha}. Our contribution is the development of a new initialization method.
When solving non-convex optimization problems, while finding good initial points is an important problem, the related literature is rather limited. The most common approach is using crude uniform sampling, which is likely suboptimal. Our approach provides a computationally feasible refined solution to this problem. Finding good initialization in more specific problem settings has been studied in the literature. For example, \cite{chen2019gradient} studies the efficacy  of gradient descent with random initialization for solving systems of quadratic equations. Weight initialization for neural networks has been investigated in \cite{hanin2018start, zhang2019fixup, ash2019warm}. Spectral initialization has been proposed for generalized linear sensing models in the high dimensional regime \citep{lu2017spectral}. The key advantage of our propose method is its general applicability and
theoretical performance guarantee. 


Our problem is related to but different from federated learning. Federated learning is a special form of distributed learning where the central learning agent do not have access to or control over individual agent's (distributed worker's) device and data \cite{li2020federated}. Most existing development in federated learning try to address two main challenges: i) the communication cost, which can be extremely high (higher than the computational cost) and ii) the agents (distributed workers) are heterogeneous where the data stored with individual agent may not be representative (non i.i.d.) (see, for example, \cite{li2019convergence,karimireddy2019scaffold, zhang2021fedpd}).  
In contrast, our setting assume the data organization owns all the data and can decide what to distribute to outsourcing computing facilities.Thus, we can ensure that the data send to individual workers are representative. 
The task we assign to individual workers is also fundamentally different from federated learning. In our setting, we divide the learning into two stages. The outsourcing stage (distributed stage) where the objective is to find good initialization and the in-house stage where we try to learn the optimal solution. We also focus on the non-convex learning setting, which is not much studied in federated learning.

Our problem is also related to but different from simulated annealing or tempering-based algorithms. Simulated annealing tries to integrate the exploration of different local minimums with the exploitation to pinpoint the global minimum \cite{kirkpatrick1983optimization}. We,
on the other hand, separate the exploration and exploitation task to two different entities. Monte Carlo simulation can be used for the exploration task in our setting. The main advantage of our algorithm is that we use sampling-based approach to find good initial point with limited data. To the best of our knowledge, this particular setting has not been discussed in the literature.

{\bf Notations.}  We use $\|\theta \|$ and $\|A\|_{\text{op}} $ to denote the $L_2$-norm of a vector $\theta$ and the operator norm of a matrix $A$ respectively. For real numbers $a,b$, let $a\wedge b=\min\{a,b\}$ and $a\vee b=\max\{a,b\}$. 
Lastly, given two sequences of real numbers $\{a_n\}_{n\ge 1}$ and $\{b_n\}_{n\ge 1}$, $a_n=O(b_n)$ denotes that there exist a constant $C>0$, such that $a_n \le C b_n$, and $a_n=\Omega(b_n)$ denotes that $a_n \ge C b_n$.

\section{Methodology} \label{methodology}
We consider minimizing a smooth but non-convex function $F(\theta)$, which takes the form
\[
F(\theta)=\mathbb{E}_{X \sim \xi}\big[f(\theta,X)\big],
\]
over a $d$-dimensional unit ball $\Theta=\{\theta \in \mathbb{R}^d: \|\theta\| \le 1 \}$. It is also common to consider an empirical loss function of $N$ data points:
$F_N(\theta)=\frac{1}{N}\sum_{i=1}^N f(\theta,x_i)$.
This can be seen as a special form of $F(\theta)$ where $\xi$ is the empirical distribution of the full dataset $\{x_1,\ldots,x_N\}$.
We also comment that in most applications, the solution to the optimization problem needs to be restricted to some known range. In
practice, we can do whitenning transformation or other rescaling so the the solution is in the unit ball. From the theoretical perspective, considering bounded domain greatly simplifies our discussion and this assumption is commonly imposed in the literature \cite{mei2018landscape}.

Since $F(\cdot)$ is non-convex, the performance of any greedy deterministic optimization algorithm relies heavily on the choice of initial points. Specifically, a deterministic optimization algorithm $\cT$ such as gradient descent (GD) or Newton's method can be trapped in a suboptimal local minimum instead of converging to the desired global minimum if initialized inappropriately. In practice, the initial points are usually sampled uniformly at random when the structure of $F(\theta)$ is unknown. Despite seeming simple and plausible, this approach lacks theoretical justification and can be highly inefficient. In this work, we design data outsourcing and exploration mechanisms to find good initial points for the optimization algorithm $\cT$. The objective is to increase the chance that $\cT$ finds the global minimum successfully.

Assume the outsourced data $\{x_1,\ldots, x_n\}$ follow the same distribution as $\xi$.
We can construct a sample approximation to $F(\theta)$ as $\hat{F}_n(\theta)=\frac{1}{n}\sum_{i=1}^n f(\theta,x_i)$.
Evaluating  $\hat{F}_n(\theta) $ or $\nabla \hat{F}_n(\theta) $ has a much smaller cost than evaluating ${F}(\theta) $ or $\nabla {F}(\theta) $ if the sample size $n$ is not too large. This  makes exploring the energy landscape of $\hat{F}_n(\cdot)$ more computationally friendly. Note that $\hat{F}_n(\cdot)$ captures certain structural information of $F(\cdot)$. We are interested in utilizing this information in an appropriate way. 
More specifically, the work of \cite{mei2018landscape} has shown that the energy landscape of $\Fhat_n(\theta)$ bears close similarity to that of $F(\theta)$ when $n$ surpasses a certain threshold. This indicates that the global minimum of $\Fhat_n(\theta)$ should  be closer to that of $F(\theta)$ than a random guess. Let $\hat{\theta}^*_0 $ denote the global minimum of $\Fhat_n(\theta)$ and $\theta_0^*$ denote the global minimum of $F(\theta)$.
Intuitively, if we use $\hat{\theta}^*_0$ as the initial point to apply the optimization algorithm $\cT$, we might be more likely to converge to $\theta^*_0$. We refer to this approach as the \emph{optimization approach}. It is quite computational friendly to the data organization, since only  one instance of $\cT$ is needed. However, it also comes with certain costs: 1) $\Fhat_n(\theta)$ is a noisy realization of $F(\theta)$, especially when $n$ is small. Using just the global minimizer of $\Fhat_n$, which is a single point, can be risky. 
2) $\Fhat_n$ is likely to be nonconvex as well and optimizing it can be expensive.  For 2), since the task is outsourced to a computing facility, the in-house cost is reduced though.


An alternative approach we consider is to sample a Bayesian posterior distribution with the outsourced data. 
In Bayesian statistics, the unknown parameter $\theta$ is usually represented using a posterior density, which is proportional to the product of a prior density and the  likelihood function. 
Since we require $\|\theta\|\in \Theta$, it is natural to assume the prior distribution is the uniform distribution on $\Theta$. In many applications, the loss function $f(\theta,x_i)$ is proportional to the negative log likelihood. For example, if we model the data output as a function of the input plus Gaussian noise, i.e. $x_{\text{out}}=g(\theta, x_{\text{in}})+\xi$ where $\xi\sim \mathcal{N}(0,\sigma^2)$, the likelihood function is given by
$-\log p(x|\theta)=\frac1{2\sigma^2}(x_{\text{out}}-g(\theta, x_{\text{in}}))^2:=\frac{1}{2\sigma^2}f(\theta,x )$.
Then, the posterior distribution is given by 
\begin{equation}
\label{eq:bayes}
p(\theta)\propto  1_{(\theta \in \Theta)}\prod_{i=1}^np(x_i|\theta)=1_{(\theta\in \Theta)} \exp(-\frac{n}{2\sigma^2} \Fhat_n(\theta)).
\end{equation}
Samples from the posterior distribution learn from $x_1,\ldots, x_n$. Thus, they are more informative than samples from the prior distribution. Comparing with the optimization approach, this \emph{sampling approach} takes into account that $\Fhat_n$ is noisy, so the candidate solution is not a single point, but a distribution which accounts for the uncertainty. In this case, the data organization needs to implement $\cT$ from multiple samples generated from the posterior distribution.

{ Given a deterministic optimization algorithm $\cT$, when only partial data is available, there is in general no clear theoretical guarantee when determining whether a point is a good initial point to optimize $F$. Both the optimization approach and the sampling approach use criteria based on $\Fhat_n$. Our theoretical analysis shows when these criteria are sufficient. We next provide more details of these two approaches. 
While the optimization approach is conceptually simpler, its computation requires sampling tools. Thus, we start with the sampling approach.} 

\paragraph{Procedures with the sampling approach}
For the exploration task, we consider sampling from a distribution
\begin{equation}\label{target}
\pi_\beta(\theta) \propto \exp(-\beta\hat{F}_n(\theta))\cdot 1_{\{\theta \in \Theta\}}.
\end{equation}
The parameter $\beta>0$ is often referred to as the
inverse temperature \citep{xu2017global}. The posterior distribution in \eqref{eq:bayes} corresponds to $\beta=\frac{n}{2\sigma^2}$.
We consider general $\beta$ because in practice the observation noise $\sigma^2$ may not be known. 
The parameter $\beta$ determines how much $\pi_\beta(\theta)$ concentrates around the global minimum of $\hat{F}_n(\theta)$. A larger $\beta$ leads to a higher concentration around $\hat{\theta}^*_0$. 
When $\beta=\infty$, we get $\hat{\theta}^*_0$ with probability one.
Using $\hat{\theta}^*_0$ as a starting point is likely to be a good choice if $n$ is large enough and $\Fhat_n$ is close to $F$. Meanwhile, 
when $\beta=0$, $\pi_\beta$ is simply the uniform distribution, which is equivalent to the standard random start. In this sense, sampling from $\pi_\beta$ with $\beta\in(0,\infty)$ can be viewed as an interpolation of two extreme cases.

There is a rich literature on how to sample from $\pi_\beta(\theta)$.
When $\pi_\beta$ is simple or close to some simple reference distributions, 
independent samples can be obtained through  rejection sampling or importance sampling.
For more complicated target distributions, Markov Chain Monte Carlo (MCMC) is usually applied. 
In general, these algorithms simulate stochastic processes of which $\pi_\beta$ is the invariant distribution. 
Popular and simple choices include random walk Metropolis, unadjusted Langevin algorithm (ULA) \cite{durmus2017nonasymptotic}, Metropolis-adjusted Langevin algorithm \cite{roberts1996exponential}.
Recent studies show that these MCMC algorithms are efficient when the sampling distribution is log-concave with perturbations \cite{dwivedi2018log,ma2019sampling}.
When $\Fhat_n$ is non-convex with separated local minima, $\pi_\beta$ is a multimodal distribution, and it can be difficult to sample directly with these algorithms.
This is particular the case if $\beta$ is large, since the stochastic algorithm may stick to one mode for many iterations before visiting the other modes. This issue can often be solved using methods such as parallel tempering or simulated tempering \cite{woodard2009conditions, ge2018simulated, tawn2020weight, dong2020spectral}. 
The papers \cite{ge2018simulated, lee2018beyond} show that a simulated tempering algorithm can sample a multimodal distribution with polynomial complexity.  

%

Given the sample $\theta_1,\ldots,\theta_L$ from $\pi_{\beta}$, the data organization then implement $\cT$ starting from each $\theta_i$. Let $\cT(\theta)$ denote the output of the optimization algorithm $\cT$ starting from $\theta$. The actual exploration algorithm is summarized in Algorithm \ref{alg:sample}. Our theoretical analysis in the next section gives rigorous justification of this procedure assuming $\beta$ is large enough. In practice, this approach is more efficient than the naive random start even with moderate $\beta$ as we will demonstrate through numerical experiment in Section \ref{sec:num}.
We also emphasize that our analysis applies to most of existing sampling tools where $\theta_1,\ldots,\theta_L$ do not need to be independent. 

\begin{algorithm}    
    \caption{Sampling-based Initial Point Selection (SIPS)}
    \begin{algorithmic}
    \STATE{{\bf Input:} 
  Outsourced data sample $\{x_1,\ldots, x_n\}$, inverse temperature parameter $\beta$, sampling algorithm $\cM$, exploration sample size $L$.} 
    \STATE{{\bf Initialization:} Construct the empirical average $\hat{F}_n(\theta)=\frac1n\sum f(\theta, x_i)$ and the target density
    $\pi_\beta(\theta) \propto \exp(-\beta\hat{F}_n(\theta))\cdot 1_{\{\theta \in \Theta\}}$.}
    \STATE{{\bf Sampling:} Apply  $\cM$ to draw samples $\{\theta_1,\ldots,\theta_L\}$ from distribution $\pi_{\beta}$.} 
    \STATE{{\bf Output:} Candidate initial points $\{\theta_1,\ldots,\theta_L\}$.}
    \end{algorithmic}  \label{alg:sample}
  \end{algorithm}

\paragraph{Procedures with the optimization approach}
%
%
When $\Fhat_n$ is non-convex, there is no consensus on how to find its global minimizer. Typical choices include either using meta-heuristic algorithms or sampling-based algorithms. Here we consider using sampling-based algorithms due to their connection to the sampling approach.

One popular way to find the global minimum of $\Fhat_n$ involves generating  samples $\theta_1,\ldots,\theta_L$ from the distribution 
$\pi_\beta$ with a large $\beta$. 
This approach is investigated by 
\cite{raginsky2017non, xu2017global, chen2020stationary} when ULA or its online version is implemented to sample from $\pi_\beta$. 
 As mentioned earlier, the parameter $\beta$ determines how much $\pi_\beta(\theta)$ concentrates around the global minimum of $\hat{F}_n(\theta)$ and a larger $\beta$ leads to a higher concentration.
When the samples $\theta_1,\ldots,\theta_L$ are available as candidate solutions, we can choose the one with the lowest 
objective value, i.e., $\theta_{i^*}$, where
\begin{align} \label{criterion_2}
i^*=\text{argmin}_{i\in\{1,\dots,L\}} \hat{F}_n\big(\theta_i\big).
\end{align} 
This procedure is summarized as the \emph{annealing} approach in Algorithm \ref{alg:optimization}. In order for this approach to be effective at finding the global minimum of $\hat{F}_n$, $\beta$ needs to be large enough. This usually increases the difficulty of sampling from $\pi_\beta$. On the other hand, it is worth noticing that we are only interested in getting good starting points for optimizing $F$. Thus, finding the global minimum of $\Fhat_n$ approximately can often serve the purpose. This suggests a less extreme $\beta$ may be sufficient.



The criterion in \eqref{criterion_2} finds the $\theta_i$ with the lowest $\Fhat_n$-value. Further refinement can be applied to improve the quality of the initial point. For example, if we apply a deterministic optimization algorithm $\hat\cT$ to $\Fhat_n$ initialized at $\theta_i$, we can achieve an even lower $\Fhat_n$-value. We then pick $\hat\cT(\theta_i)$ with the lowest $\Fhat_n$-value as the initial point, i.e., $\hat\cT(\theta_{i^*})$, where
\begin{equation} \label{criterion_3}
i^*=\text{argmin}_{i\in\{1,\dots,L\}} \hat{F}_n\big(\hat\cT(\theta_i)\big).
\end{equation}
This procedure is summarized as the \emph{sampling-assisted-optimization (SAO) approach} in Algorithm \ref{alg:optimization}. The SAO approach is similar to GDxLD developed in \cite{dong2021replica}. Comparing to the simpler criterion \eqref{criterion_2}, sampling for \eqref{criterion_3} can often be done more efficiently. This is because when implementing SAO for $\Fhat_n$, we separate the exploration task and the optimization task. This allows us to use a smaller $\beta$ when sampling $\pi_{\beta}$. The cost is that invoking $\hat\cT$ to each sample in SAO can impose extra computational cost than the annealing approach.  In contrast, the annealing approach combines the exploration task with the optimization task. So a larger $\beta$ is needed in general, which increases the cost to sample from $\pi_\beta$.

\begin{algorithm}   \   
    \caption{Optimization-based Initial Point Selection (OIPS)}
    \begin{algorithmic}
    \STATE{{\bf Input:} 
  Outsourced data sample $\{x_1,\ldots, x_n\}$, inverse temperature parameter $\beta$, exploration sample size $L$, sampling algorithm $\cM$, optimization algorithm $\hat{\cT}$.} 
    \STATE{{\bf Initialization:} Construct the empirical average $\hat{F}_n(\theta)=\frac1n\sum f(\theta, x_i)$ and the target density
    $\pi_\beta(\theta) \propto \exp\{-\beta\hat{F}_n(\theta)\}\cdot 1_{\{\theta \in \Theta\}}$}
    \STATE{{\bf Sampling:} Apply  $\cM$ to draw a sample $\{\theta_1,\ldots,\theta_L\}$ from distribution $\pi_{\beta}$.} 
    \IF{Annealing} 
    \STATE{ Set $\theta^{0}=\theta_{i^*}$ where  $i^*=\text{argmin}_{i\in\{1,\dots,L\}} \hat{F}_n\big(\theta_i\big).$}
    \ENDIF
    \IF{Sampling-assisted-optimize (SAO)} 
    \STATE{Set $\theta^0=\hat\cT(\theta_{i^*})$ where $i^*=\text{argmin}_{i\in\{1,\dots,L\}} \hat{F}_n\big(\hat\cT(\theta_i)\big).$}
    \ENDIF
    \STATE{{\bf Output:} Candidate initial points $\theta^0$}
    \end{algorithmic}  \label{alg:optimization}
  \end{algorithm}

\section{Theoretical guarantee in  finding the global minimum}

In this section, we analyze the performance of Algorithms \ref{alg:sample} and \ref{alg:optimization}. 
The key in successful implementation of the algorithms is to set the appropriate outsourcing sample size $n$, inverse temperature $\beta$, and exploration sample size $L$. Our performance analysis provides guidelines on choosing these parameters. 

\paragraph{Conditions on the energy landscapes. }
We start with some assumptions on the energy landscape of $F(\theta)$ and the randomness when evaluating $f(\theta,x)$.
Many of them are also assumed in \cite{mei2018landscape}.
Since we run an optimization algorithm $\cT$ that converges to a stationary point in the second phase, 
the following assumption regularizes the configuration of the stationary points:
\begin{assumption} \label{aspt:morse}  $F(\theta): \Theta \to \mathbb{R}$ is $(\sigma,\eta)$-strongly Morse, that is, $\|\nabla F(\theta) \|\geq\sigma$ for $\|\theta\|=1$, and $\lambda_{\min}(\nabla^2 F(\theta))\geq \eta \text{ if }\|\nabla F(\theta) \| \le \sigma$, where $\lambda_{\min}(A)$ is the minimum eigenvalue of $A$.
Moreover $L^*:=\sup_{\theta \in \Theta}\|\nabla^3 F(\theta) \|_{\text{op}}<\infty$.
\end{assumption}
One consequence of Assumption \ref{aspt:morse} is that that all the stationary points of $F(\theta)$ in $\Theta$ are finite and well-separated \citep{mei2018landscape}.
In particular, we can denote these stationary points as $({\theta}^*_0,{\theta}^*_1,\ldots{\theta}^*_K)$. Without loss of generality, let ${\theta}^*_0$ be the global minimum of $F(\theta)$. 
%

For simplicity of discussion, we assume that  $\cT$ is a deterministic optimization algorithm that is guaranteed to converge to a stationary point and the performance of $\cT$ is determined by the initial point. Starting from $\theta^0$, we denote the stationary point that $\cT$ converges to as $\cT(\theta^0)$. Hence, $\cT$ can be viewed as a deterministic mapping from the parameter space $\Theta$ to the set of stationary points $\{\theta_0^*,{\theta}^*_1,\ldots,{\theta}^*_K\}$. Our goal is to find a $\theta^0$ such that  $\cT(\theta^0)=\theta^*_0$.  

Given the deterministic optimization algorithm $\cT$, the attraction region of the global minimum $\theta^*_0$ can be defined as
 \[
 \B_0^* = \{ \theta \in \Theta: \cT(\theta)=\theta_0^*\}.
 \]
{In general, $\B_0^*$ cannot be characterized without $\cT$.} On the other hand, it is well-known that for many optimization algorithms, $\cT(\theta^0)=\theta_0^*$ if $\theta^0$ is in a neighborhood of $\theta_0^*$ in which $F(\theta)$ is strongly convex. This indicates that a proper neighborhood of $\theta_0^*$ can be used as a substitution of $\B_0^*$. We formalize this idea as follows. 
{\begin{assumption} \label{ass_ab}
There exists a ball centered at $\theta_0^*$ with radius $r$, $\cB_r(\theta_0^*)=\{ \theta: \|\theta-\theta^*_0 \| \le r \}  $, such that $\cB_r(\theta_0^*) \subseteq \B_0^*$ and $F(\theta)$ is $\mu$-strongly convex in $\cB_r(\theta_0^*)$.
\end{assumption}
 Note that Assumption \ref{ass_ab} may comes as a consequence of Assumption \ref{aspt:morse}. In particular, $F(\theta)$ is $ \eta/2$-strongly convex in  $\cB_r(\theta_0^*)$ when $r\leq \frac{\eta}{2L^*}$. }

 Notably, the assumptions above  enable us to derive an upper bound for the failure rate of the benchmark random start algorithm. 
If we draw $m$ independent initial points from $\Theta$ uniformly at random, the probability that none of them leads to the global minimum of is  $\Pb(\cF_b)\leq (1-\Pb(\cB_r))^m$, or 
\begin{equation}
\label{eq:random}
\log \Pb(\cF_b)\leq -L\log (\Pb(\cB_r))=\Omega(Lr^d).
\end{equation} 
Then, in order for it to be lower than a threshold $\rho$, we need $m=\log(\rho)/|r|^d$, which has an exponential dependence on $d$.
 
%
%
Our next assumption concerns the uniqueness of the global minimum.
When $\theta^*_0$ is the unique global minimum, its function value needs to be strictly lower than the other stationary points. 
{
\begin{assumption} \label{ass1} 
There exists a constant $\alpha > 0$, such that  for all $\theta \notin \cB_r(\theta_0^*)$,
$F(\theta) - F(\theta_0^*) \ge \alpha$.
\end{assumption}
}
In Section \ref{sec:epsglobal}, we will discuss what can be achieved if this assumption does not hold.

The basic idea of our data outsourcing and exploration scheme is to approximate $F(\theta)$ via its sample average $\hat{F}_n(\theta)$ and then use the global minimum of $\hat{F}_n(\theta)$ as the initial point to optimize $F(\theta)$. 
A key question is that in order for $\Fhat_n(\theta)$ to be a good approximation of $F(\theta)$, how many data points are needed. 
This problem has been studied \cite{mei2018landscape}. We adapt some of their results into our setting. 
This involves the following regularity conditions on the loss function and noises (similar versions of them can be found in \cite{mei2018landscape} as well). 
\begin{assumption}
\label{ass_c0}
The following hold for some $\tau, c_h$
\begin{enumerate}
\item The loss function for each data point is $\tau^2$-sub-Gaussian. Namely, for any $\lambda \in \R^p $, and $\theta \in \Theta$, 
    \begin{align*}
        \E \Big[ \exp \left(\big \la  \lambda, f(\theta;X)-\E_{X\sim \xi}[f(\theta;X)] \big \ra \right)  \Big ] \le \exp \Big\{ \frac{\tau^2\|\lambda\|^2}{2}\Big\}.
    \end{align*}
\item The gradient of the loss is $\tau^2$-sub-Gaussian. Namely, for any $\lambda \in \R^p $, and $\theta \in \Theta$, 
    \begin{align*}
        \E \Big[ \exp \big \la  \lambda, \nabla_\theta f(\theta;X)-\E_{X\sim \xi}[\nabla_\theta f(\theta;X)] \big \ra  \Big ] \le \exp \Big\{ \frac{\tau^2\|\lambda\|^2}{2}\Big\}.
    \end{align*}
\item The Hessian of the loss, evaluated on a unit vector, is $\tau^2$ -sub-exponential. Namely, for any $\| \lambda\|\le 1$, and $\theta \in \Theta$,  
\begin{align*}
    \E \Big[ \exp \Big\{ \frac{1}{\tau^2} \big|\cZ_{\lambda,\theta}(X)-\E_{X\sim \xi}[\cZ_{\lambda,\theta}(X)]  \big| \Big\}  \Big ] \le 2,
\end{align*}
where 
$\cZ_{\lambda,\theta}(X)=\la  \lambda, \nabla^2_\theta f(\theta;X)\lambda \ra$.
\item 
There exists $J_*$ (potentially diverging polynomially in $d$) such that
\begin{align*}
    \E_{X\sim \xi} \big[ J^1(X) \big], \E_{X\sim \xi} \big[ J^2(X) \big] \le J_*
\end{align*}
where 
\begin{align*}
    J^1(X)&=\sup_{  \theta_1,\theta_2 \in \Theta, \theta_1\ne\theta_2} \frac{\| \nabla_\theta f(\theta_1;X)- \nabla_\theta f(\theta_2;X) \| }{\|\theta_1-\theta_2\|},\\
     J^2(X)&=\sup_{  \theta_1,\theta_2 \in \Theta, \theta_1\ne\theta_2} \frac{\| \nabla_\theta^2f(\theta_1;X)- \nabla_\theta^2f(\theta_2;X) \|_{\text{op}}    }{\|\theta_1-\theta_2\|}.
\end{align*}
Furthermore, there exists a constant $c_h$ such that $J_*\le \tau^3 d^{c_h}$. 
\item There exists $\theta^* \in \Theta$, such that  $\|\nabla F(\theta^*)\|,  \|\nabla^2 F(\theta^*)\|_{\text{op}}\le H\leq  \tau^3 d^{c_h}$. 
\end{enumerate}
\end{assumption}

Assumption \ref{ass_c0} allows us to find a close approximation of $F$, which is formally defined as follows.
\begin{definition}
We say $\Fhat_n(\theta)$ is a $\delta$-approximation of $F(\theta)$, if both $F$ and $\hat{F}_n$ have $K+1$ stationary points, denoted by $\{\theta^*_i\}_{i=0,\ldots,K}$ and $\{\hat{\theta}^*_i\}_{i=0,\ldots,K}$, and the following inequalities hold
 \begin{align*}
        &\sup_{\theta  \in \Theta} |F(\theta)-\hat{F}_n(\theta)|\leq \delta,\  \sup_{\theta  \in \Theta} \| \nabla F(\theta)-\nabla \hat{F}_n(\theta) \|\leq \delta,\\
        &\sup_{\theta  \in \Theta} \| \nabla^2 F(\theta)-\nabla^2 \hat{F}_n(\theta) \|_{\text{op}} \le \delta,\ \mbox{ and }  \max_{0\le i \le K}\|\theta_i^*-\hat{\theta}_i^*\|  \le \delta.
   \end{align*}
 \end{definition}

The next lemma characterizes the minimal sample size required to achieve a $\delta$-approximation. 

\begin{lemma} \label{lem1}
    Assume that Assumptions \ref{aspt:morse} and \ref{ass_c0} hold. Consider a given confidence level $\rho \in (0,1)$ and a given accuracy $\delta$. Let $C=C_0\cdot (c_h\vee 1\vee \log(\tau/\rho))=O(|\log \rho|)$, where $C_0$ is some absolute constant, and $\eta_*=( \sigma^2/\tau^2)\wedge(\eta^2/\tau^4)\wedge(\eta^4/((L^*\tau)^2)=\Omega(1)$. For an arbitrary constant $\iota>0$, let $C_\iota$ be a constant such that  $\log(n)\le C_\iota \cdot n^\iota$. Then, when
\begin{align*}
n &\ge \max\Big\{\Big[ \frac{C_\iota Cd}{(\delta/((2\tau/\eta)\vee\tau \vee \tau^2))^2} \Big]^{\frac{1}{1-\iota}}, 4Cd \big(\log(d/)\vee \log(n)/\eta^2_*\big)\Big\},\\
& :=n(\delta,\rho,d),
\end{align*}  
with probability at least $1-\rho$, $\Fhat_n(\theta)$ is a $\delta$-approximation of $F(\theta)$. 
\end{lemma}

The proofs of Lemma \ref{lem1} and all subsequent results are provided in Appendix \ref{sec:proof}.

Lemma \ref{lem1} quantifies that to achieve a $\delta$-approximation of $F(\theta)$ with confidence level $1-\rho$, the required sample size is 
\begin{align}  \label{sample_size}
n(\delta,\rho,d)=O\Big(\frac{d\log(1/\rho)}{\delta^2}\Big).
\end{align}
Here, we ignore the index $\iota$ in the power since it can be made arbitrarily small. 

\subsection{Performance of the sampling approach}
We denote $\cF_0$ as the event that using the initial point(s) constructed based on Algorithm \ref{alg:sample}, the optimization algorithm in the exploitation stage fails to find the global minimum. In this section, we establish an upper bound for $\Pb(\cF_0)$.

Recall that samples are drawn from the distribution 
$\pi_{\beta}$ defined in \eqref{target}.
We will justify that when $\beta$ is large enough, a random sample $\tilde{\theta}_\beta$ from $\pi_\beta(\theta)$ is a good starting point to optimize $F(\theta)$. 
In particular, $\tilde{\theta}_\beta$ has a high chance to fall into an attraction basin of $\theta^*_0$, i.e., $\cB_r(\theta^*_0)$. 

\begin{proposition}  \label{lem2}
    Suppose Assumptions \ref{aspt:morse}-\ref{ass_c0} hold and the approximation accuracy $\delta$ satisfies $\delta<\mu \wedge r \wedge \alpha/4$. If $\hat{F}_n(\theta)$ is a $\delta$-approximation of $F(\theta)$ and $\beta \ge \Omega(r^{-2})$, then the probability that $\tilde{\theta}_\beta$ fails to be a good starting point, i.e. $\pi_\beta\big(\cB^c_r(\theta^*_0)\big)$, is bounded by:
\[
    \log\big(\pi_\beta\big(\cB^c_r(\theta^*_0)\big)\big)=  O\big( -\beta\alpha/2+d\log(\beta) \big).
\]
\end{proposition}

Proposition \ref{lem2} shows that as the inverse temperature parameter $\beta$ increases, the probability that we can sample points from $\cB_r(\theta^*_0)$ approaches one exponentially fast. The convergence speed is determined by $\alpha$, the gap between the global minimum and other local minima, as well as the dimension parameter $d$. However, in practice, we cannot choose $\beta$ arbitrarily large as we have to consider the computational cost in associated sampling algorithms (e.g., an MCMC algorithm). In general, when $\beta$ increases, the difficulty of sampling from $\pi_\beta(\theta)$ increases. In practice, we want to find a $\beta$ that balances the estimation accuracy and the sampling difficulty.

One difficulty when applying Proposition \ref{lem2} to the sampling approach is that in practice we may not be able to sample from $\pi_\beta$ exactly. For example, many MCMC algorithms can only draw samples from a distribution that is ``close" to $\pi_\beta$. To handle this issue,  To handle this issue, we impose the following assumption as a relaxation to the requirement of sampling from $\pi_{\beta}$ exactly.

 \begin{assumption} \label{ass2_v2}
There is a sampler $\hat{\cM}$ such that for any fixed $\delta_\beta\in [0,1)$, starting from any $\theta_0\in\Theta$, $ {\cM}$  can draw samples from a distribution $\hat{\pi}_{\beta}$ which satisfies $\|\hat{\pi}_{\beta}-\pi_\beta  \|_{TV}\le \delta_\beta.$
\end{assumption}

In addition, note that in practice, 
we can draw consecutive samples from the same chain of the underlying MCMC algorithm, which makes the samples correlated. 
The following lemma justifies the quality of the samples form $\hat{\cM}$ under Assumption \ref{ass2_v2}.

\begin{lemma} \label{approximation_sampling}
    Given a set $B$ and distribution $\pi_\beta$ with $\pi_\beta(B)>0$, suppose there exists a samplers $\hat{\cM}$ satisfying Assumption \ref{ass2_v2}. If we have $L$ samples from $\hat{\cM}$, then
    \[
      \Pb(X_1\notin B,\ldots, X_L\notin B)\leq (\pi(B^c)+\delta_\beta)^L.
    \]    \end{lemma}

The following theorem then comes as a consequence of Proposition \ref{lem2} and Lemma \ref{approximation_sampling}.
 
\begin{theorem} \label{thm0}
Consider Algorithm \ref{alg:sample}. Suppose Assumptions \ref{aspt:morse}-\ref{ass2_v2} hold. For an arbitrary confidence level $\rho \in (0,1)$, let  $\delta=\mu \wedge r \wedge \alpha/4 $. If  sample size $n\geq n(\delta,\rho,d)=O(d\log(1/\rho)/\delta^2)$ and the inverse temperature $\beta \ge \Omega({r^{-2}})$, then  
there exists a constant $C\in(0,\infty)$ such that
\begin{equation}
  \label{equ:fail}
\Pb(\cF_0)\leq \rho+\exp \big(CL\cdot \max\big\{-\beta\alpha/2+d\log(\beta),\log(2\delta_\beta)\big\}\big).
   \end{equation}
   
\end{theorem}


Theorem \ref{thm0} shows that the probability that the sampling approach fails to find the global minimum of $F(\theta)$ decays exponentially fast as the inverse temperature $\beta$ and sample size $L$ increase.
In particular, $L$ only needs to surpass some dimensional independent constants, i.e., the convexity constant $\mu$ and the separability constant $\alpha$ of global minimum from other local minima.  In contrast, by \eqref{eq:random}, the benchmark random start method would require the number of random initialization $m$ to depend exponentially on the dimension. 
We comment that Algorithm \ref{alg:sample} does require an outsourced sampling algorithm to obtain samples from $\pi_\beta$, which can be computationally costly, 
but this task is outsourced and we achieve a much smaller the in-house computational cost.
Finally, it is worth mentioning that in Theorem \ref{thm0}, both $\beta$ and $n$ scale as $r^{-2}$. This is in agreement with the Bayesian setup  \eqref{eq:bayes}, which suggests $\beta$ should scale linearly with $n$. 




\subsection{Performance of the optimization approaches}
We first provide an analysis of the SAO approach in Algorithm \ref{alg:optimization}. 
Let $\cF_1$ denote the random event that the output of Algorithm \ref{alg:optimization}-SAO approach fails to find the global minimum of $F(\theta)$. 
The result is largely the same as Theorem \ref{thm0}, although the proof is slightly more difficult. 
\begin{theorem}  \label{thm1}  
Consider Algorithm \ref{alg:optimization}-SAO. Suppose Assumptions \ref{aspt:morse}-\ref{ass2_v2} hold. For an arbitrary confidence level $\rho \in (0,1)$, let  $\delta=\mu \wedge r \wedge \alpha/4 $. If the sample size $n\geq n(\delta,\rho,d)=O(d\log(1/\rho)/\delta^2)$ and the inverse temperature $\beta \ge \Omega({r^{-2}})$, then   
there exists a constant $C\in(0,\infty)$ such that
\[
\Pb(\cF_1)\leq \rho+\exp \big(CL\cdot \max\big\{-\beta\alpha/2+d\log(\beta),\log(2\delta_\beta)\big\}\big).
\]

\end{theorem}

We next analyze the annealing approach in Algorithm \ref{alg:optimization}. Let $\cF_2$ be the random event that Algorithm \ref{alg:optimization}-annealing fails to find the global minimum of $F(\theta)$.
The annealing approach needs more restrictions than the SAO approach. This is because: in order to generate a good starting point, one of the samples need to fall close to $\theta^*_0$. Moreover, its $\Fhat_n$-value needs to be lower than the other samples. This 
can be formulated as requiring a smaller radius $r_0$ for the attraction neighborhood:
%
\begin{theorem} \label{thm2}
Consider Algorithm \ref{alg:optimization}-annealing. Suppose Assumptions \ref{aspt:morse}-\ref{ass2_v2} hold. For an arbitrary confidence level $\rho \in (0,1)$, let  $r=r_0$ where $r_0^2\cdot \sup_{\theta \in \Theta}\|\nabla^2 F(\theta)  \|_{\text{op}}  <\alpha$ and $\delta=\mu \wedge r \wedge \alpha/4 $.  If the sample size $n\geq n(\delta,\rho,d)=O(d\log(1/\rho)/\delta^2)$ and the inverse temperature $\beta \ge \Omega({r^{-2}})$, then   
there exists a constant $C\in(0,\infty)$ such that
\[
\Pb(\cF_2)\leq \rho+\exp \big(CL\cdot \max\big\{-\beta\alpha/2+d\log(\beta),\log(2\delta_\beta)\big\}\big).
\]
\end{theorem}

\subsection{Extension to $\epsilon$-Global Minimum}
\label{sec:epsglobal}
One major constraint in our previous analysis is Assumption \ref{ass1}--the global minimizer is unique with a gap of $\alpha>0$. In practice, there can be multiple local minima that have function values very close to the global minimum.
In this setting, it can be too ambitious to fine the global minimum and it may be more reasonable to find an approximately optimal solution. 
%
%
Given a user-specified accuracy level $\epsilon$, we are interested in finding a local minimum whose objective value is within $\epsilon$-distance from the optimal objective value, i.e., $\theta^*_i$ such that $F(\theta_i^*)\le F(\theta_0^*)+\epsilon$. We call a such local minimum an $\epsilon$-global minimum of $F(\theta)$. In this subsection, we conduct performance analysis for our algorithms to find an $\epsilon$-global minimum. Let
\[\cJ^*_\epsilon=\big\{i: F(\theta_i^*)\le F(\theta_0^*)+\epsilon\big\}\] 
be the index set of the $\epsilon$-global minimums. To be concise, we only present the analysis for the annealing-based optimization approach (Algorithm \ref{alg:optimization}-annealing). The results for the other methods are similar. 

We first introduce the ``attraction region'' of the $\epsilon$-global minimums:
\begin{definition}[``Attraction region'' of $\epsilon$-global minimums] Given an optimization algorithm $\cT$, we define the attraction basin of $\epsilon$-global minimums of $F(\theta)$ as
 $$
 \B_\epsilon^* = \big\{ \theta \in \Theta: F\big(\cT(\theta)\big)\leq F(\theta_0^*)+\epsilon \big\}.
 $$
\end{definition}
By definition, the optimization algorithm $\cT$ converges to an $\epsilon$-global minimum if and only if it starts with an initial point in $ \B_\epsilon^* $. However, same as before,  $ \B_\epsilon^* $ is hard to characterize directly. So we consider the following subset as a substitution
$$
\cB_{\epsilon,r_\epsilon}:=\bigcup_{i \in \cJ^*_\epsilon} \cB_{r_\epsilon}(\theta^*_i) \subseteq  \B_\epsilon^*.
$$
  
Let $\cF_{\epsilon,2}$ be the random event that the output of Algorithm \ref{alg:optimization}-annealing fails to find the $\epsilon$-global minimum of $F(\theta)$. 
\begin{theorem} \label{thm3}
    Suppose Assumptions \ref{aspt:morse}, \ref{ass_c0} and \ref{ass2_v2} hold. For any user-specified accuracy $\epsilon>0$, pick $r_\epsilon$ such that $\sup_{\theta \in \Theta} \|\nabla^2 F(\theta)\|_{\text{op}} \cdot r_\epsilon^2 \le \epsilon$. In addition, assume the approximation accuracy $\delta$ satisfies $\delta<\mu \wedge r_\epsilon \wedge \epsilon/12$. For an arbitrary confidence level $\rho \in (0,1)$, if the sample size $n\ge n(\delta,\rho,d)=O(d\log(1/\rho)/\delta^2)$ and the inverse temperature $\beta \ge \Omega({r_\epsilon^{-2}})$, then there exists a constant $C\in(0,\infty)$ such that
\[
\Pb(\cF_{\epsilon,2})\leq \rho+\exp \big(CL\cdot \max\big\{ -\beta\epsilon/6+d\log(\beta),\log(2\delta_\beta) \big\}\big).
\]
\end{theorem}

Note that Theorem \ref{thm3} establishes a similar performance guarantee to Theorem \ref{thm2}. However, the convergence rate in Theorem \ref{thm3} is determined by the user-specified accuracy $\epsilon$ instead of the gap constant $\alpha$.



\section{Numerical Experiment} \label{sec:num}
 In this section, we conduct numerical experiments to demonstrate the performance of our exploration and data outsourcing mechanisms.
We compare the performance of our algorithms to random start. We also run sensitivity analysis to demonstrate the robustness of our algorithms with respect to two key hyper-parameters: the outsourcing sample size $n$ and the inverse temperature $\beta$.
 
\subsection{Classic Nonconvex Test Function}
We first consider a classic nonconvex optimization problem -- the Styblinski-Tang function (ST-function) \citep{grigoryev2016global}. 
A $d$-dimensional ST-function is defined as  
\begin{align*}
    F(\theta)=\frac{\sum_{i=1}^d [\theta]^4_i-16[\theta]^2_i+5[\theta]_i}{2d},\ -5\le [\theta]_i \le 5,
\end{align*}
where $[\theta]_i$ denotes the $i$-th coordinate of $\theta$. Note that  ST-function is additively separable. By the first-order optimality condition, the stationary point set of $F(\theta)$ is $\{\theta \in \mathbb{R}^d: 4[\theta]^3_i-32[\theta]_i+5=0, \forall i \in [d] \}$. Moreover, the unique global minimum of ST-function is $\theta^*\approx (-2.903,\ldots,-2.903)$ and the corresponding objective value is $-39.165$.     
In this numerical experiment,  we set $d=5$ and use gradient descent (GD) as $\cT$.
We apply OIPS-annealing to generate the initial points. Since the ST-function is not defined through expectation, we do not consider data outsourcing here, i.e., $\hat{F}_n(\theta)=F(\theta)$.  We test different inverse temperatures $\beta=1,4$, and $10$. For each $\beta$, we use importance sampling to draw i.i.d. samples from the target distribution $\pi_\beta(\theta)\propto \exp\{-\beta F(\theta)\}$ exactly. For GD in the optimization phase, we use a step-size $0.05$ and run $50$ iterations. We pick the objective value at the last iteration as the convergent value. As the benchmark, we sample the initial point uniformly at random from the cubic $[-5,5]^d$ (random start). Finally, for each setting, we repeat the procedure $500$ times and record the final convergent values. Figure \ref{exp1} shows the distribution of convergent function values when the initial points are drawn from OIPS-annealing algorithm with different values of $\beta$ versus the benchmark method.  Note that compared with random start, initial points obtained by OIPS-annealing typically lead to smaller objective values. Moreover, as the inverse temperature $\beta$ increases, the performance of OIPS-annealing algorithm further improves.

\begin{figure}[ht]
    \centering  
    \subfigure[Random start]{
    \includegraphics[width=0.33\textwidth]{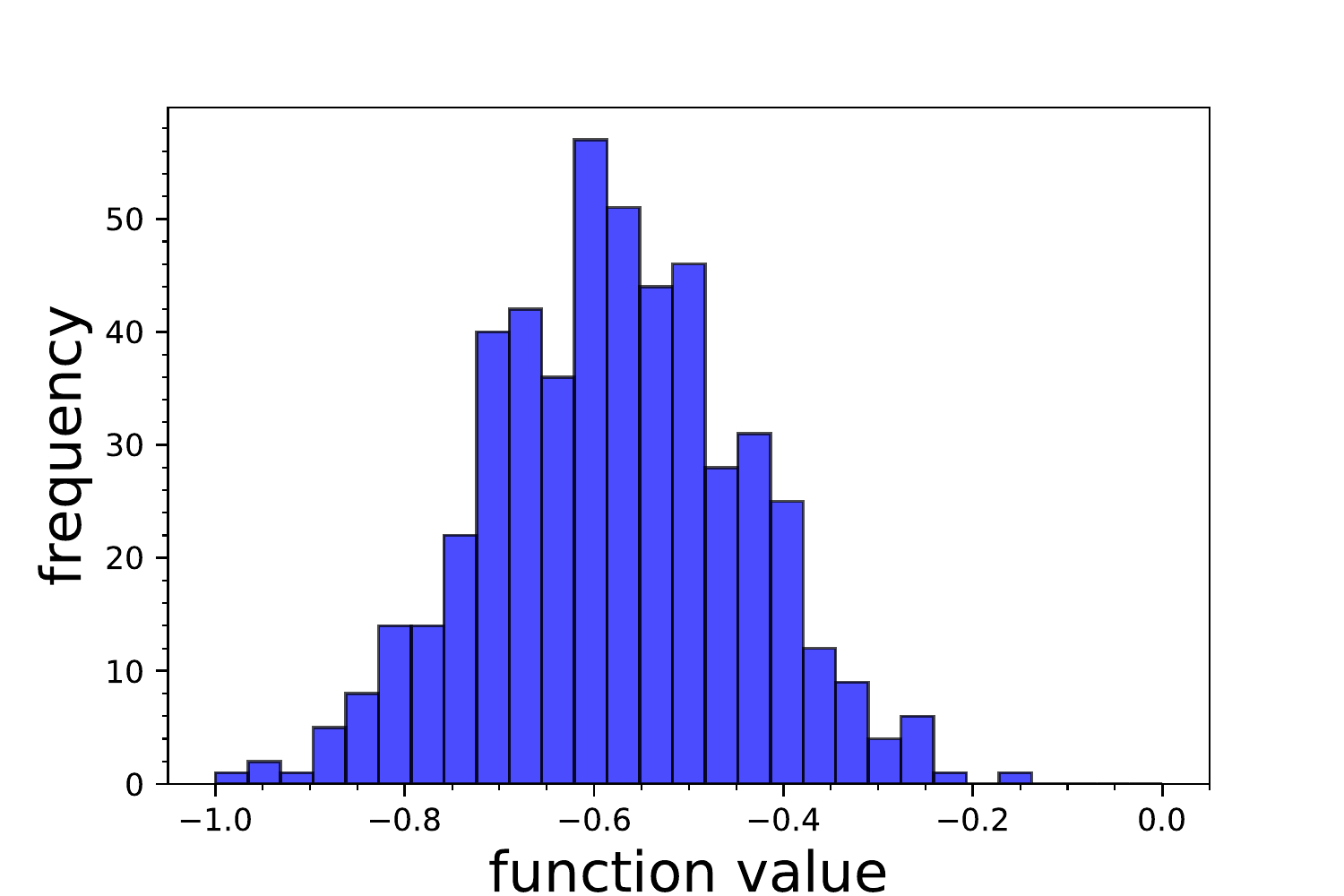}}
    \subfigure[$\beta=1$]{
    \includegraphics[width=0.33\textwidth]{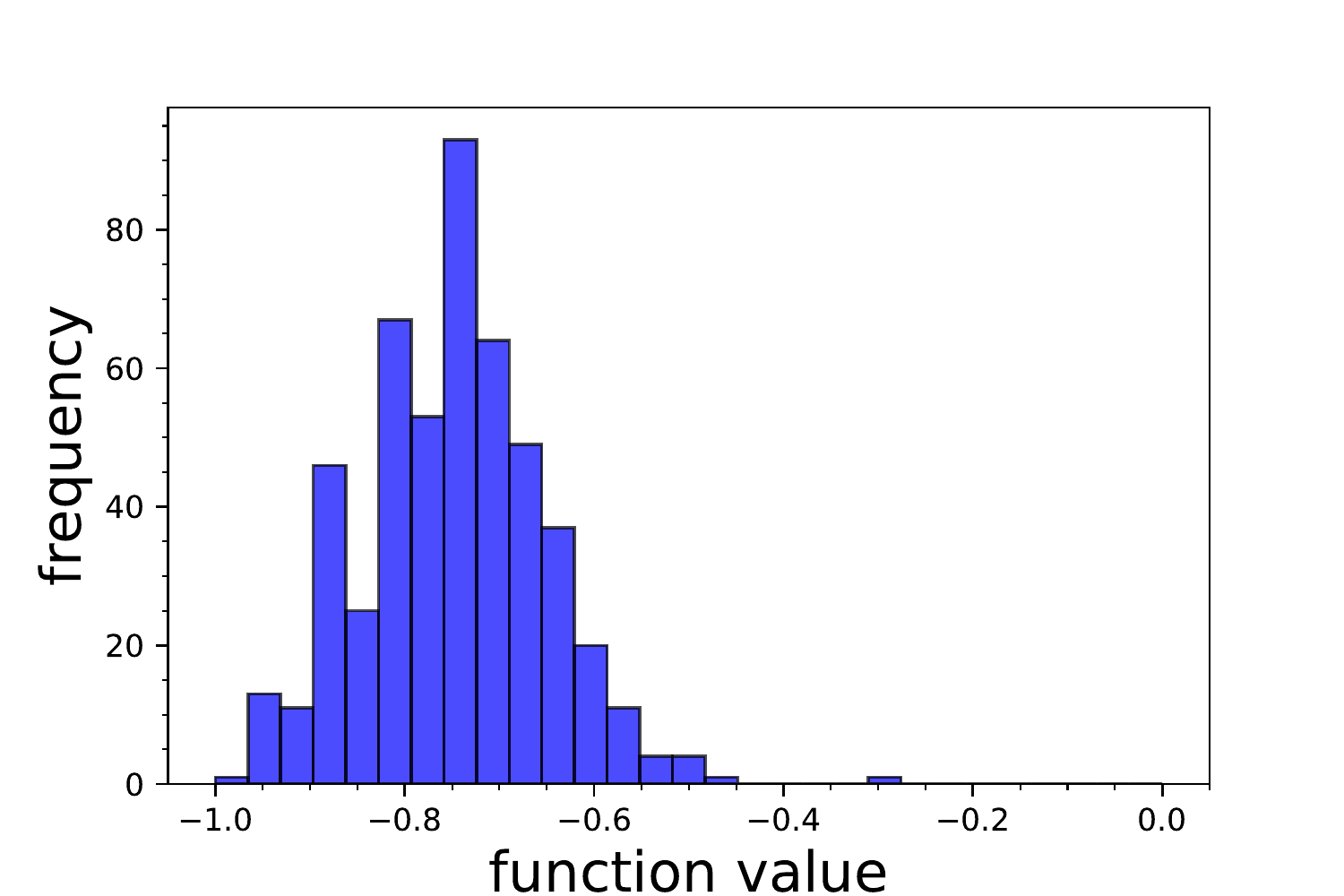}}\\
    \subfigure[$\beta=4$]{
        \includegraphics[width=0.33\textwidth]{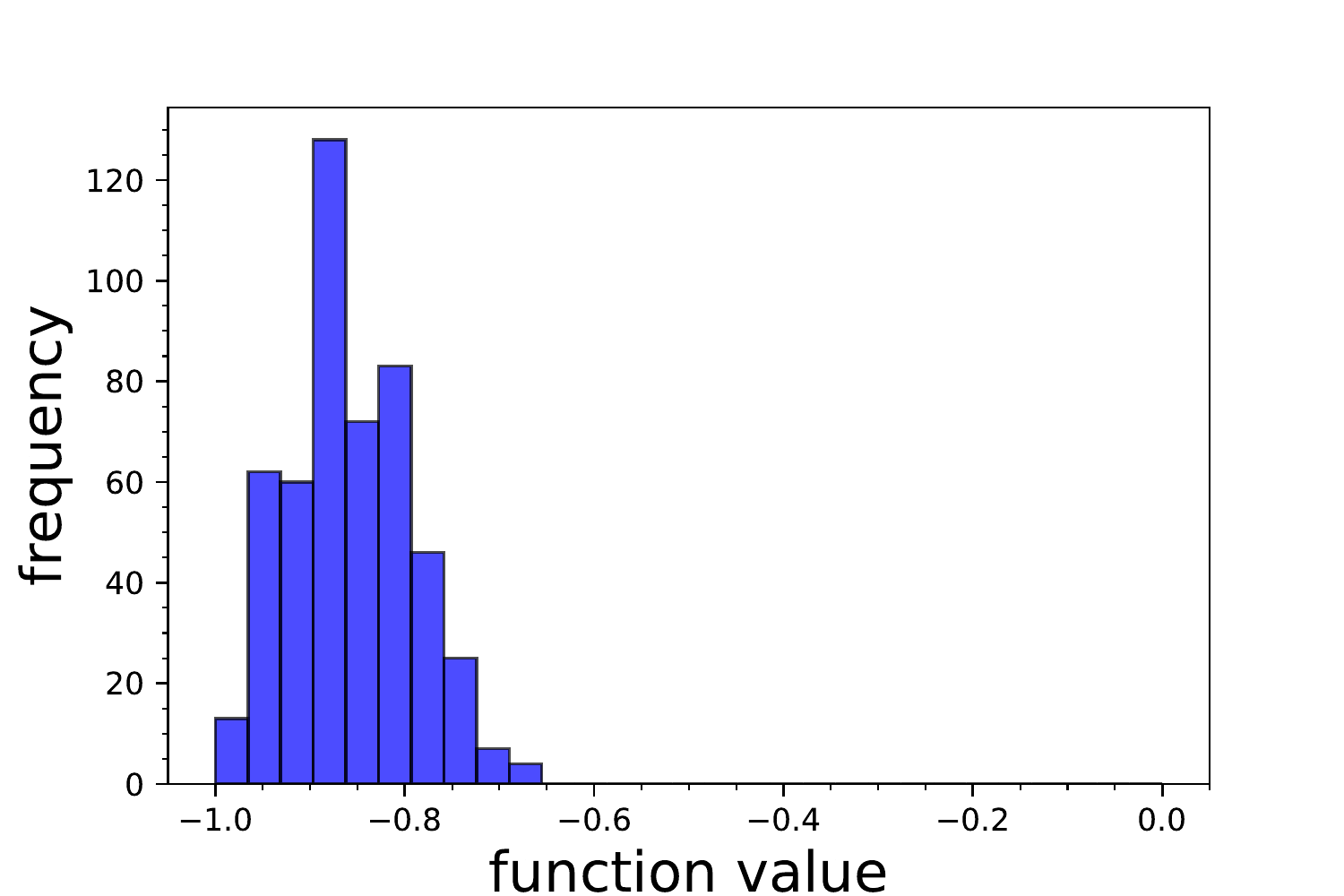}}
    \subfigure[$\beta=10$]{
    \includegraphics[width=0.33\textwidth]{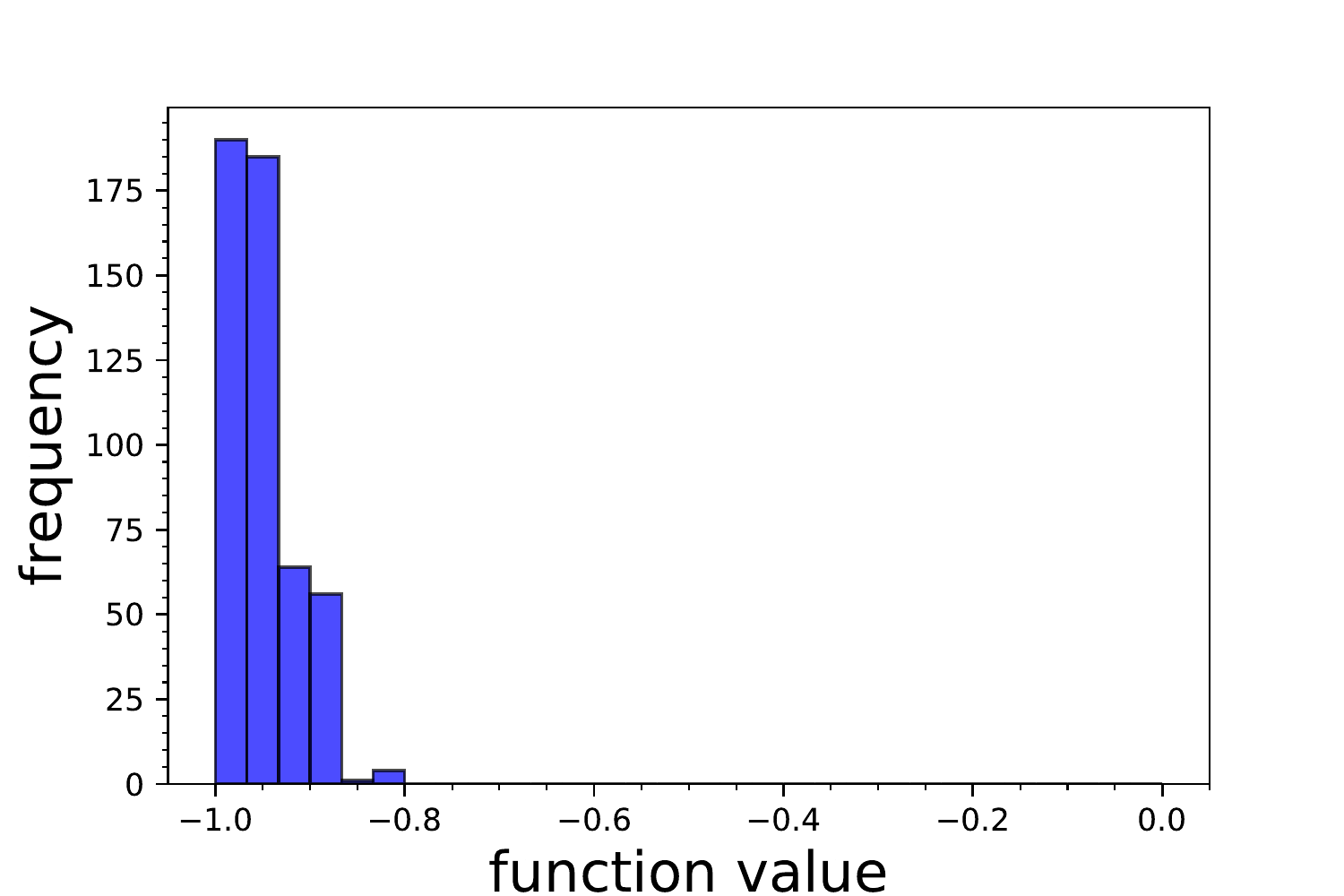}}
    \caption{Histogram of convergent function values (divided by  $39.165$) of ST-function}
    \label{exp1}
\end{figure}

\subsection{Gaussian Mixture Density}
We study the problem of finding the largest mode of a Gaussian mixture density using kernel density estimaiton. 
In particular, the objective function
\[    
F(\theta)=\E_{X \sim \xi}\Big[ {(2\pi\sigma)^{-d/2}}\cdot \exp\Big\{-\frac{\|\theta-X\|^2}{2\sigma^2} \Big\} \Big].
\]
We assume $\xi$ is a Gaussian mixture distribution, that is, 
$X\sim\mathcal{N}(m_i,\sigma^2I_d)$ with probability $p_i$ for $1\leq i\leq M$, where $\mathcal{N}(m_i,\sigma^2I_d)$ denotes the Gaussian distribution with mean vector $m_i$ and covariance matrix $\sigma^2I_d$; the mixing weights $p_i$ satisfy $0<p_i<1$ and $\sum_{i=1}^M p_i=1$. When $m_i$'s are well-separated, $F(\theta)$ has multiple local minima located near $m_i$. Hence, the selection of initial point is critical to optimize $F(\theta)$. 

We first consider a lower dimensional example with $d=5$ and $M=10$.  We implement SIPS, OIPS-annealing, and OIPS-SAO, all with $n=50$ and $\beta=10$. ULA is used to draw $L=1000$ samples from $\pi_\beta(\theta)$. Given the initial point, GD is used to optimize $F(\theta)$. Moreover, to evaluate the gradient, we draw a batch $(X_1,\ldots,X_{1000})$ from $\xi$ and approximate  $\nabla F(\theta)$ via batch means. In the optimization phase, GD is run for $20$ iterations and the objective value at the last iteration is taken as the convergent value. 
Again, $500$ independent replications of the algorithm are implemented in each setting. Figure \ref{exp2} shows the distribution of convergent function values under different algorithms  (number in bracket: success probability $\Pb(\cF^c_0)$). We observe that SIPS and OIPS outperform random start significantly. SIPS and OIPS-SAO perform better than OIPS-annealing with SIPS performs the best as measured by the probability of convergent function values smaller than $-32$. However, OIPS-annealing is the easiest and cheapest to implement in practice.
Figure \ref{exp2.5} further illustrates the success probability for different values of $n$ in OIPS-annealing and SIPS. We observe that there is a diminishing return in the outsourcing sample size. The sample sizes that are larger than $50$ in OIPS-annealing or even $30$ in SIPS lead to similar performances.

We also consider a higher dimensional example with $d=30$ and $M=20$. We focus on OIPS-annealing versus random start because of the relatively low computational cost of OIPS-annealing.  We adopt the same hyper parameters as above.  Figure \ref{exp2_2} presents the results. We observe again that OIPS outperforms random start significantly.

\begin{figure}[ht]
    \centering 
    \subfigure[Random start]{ 
    \includegraphics[width=0.33\textwidth]{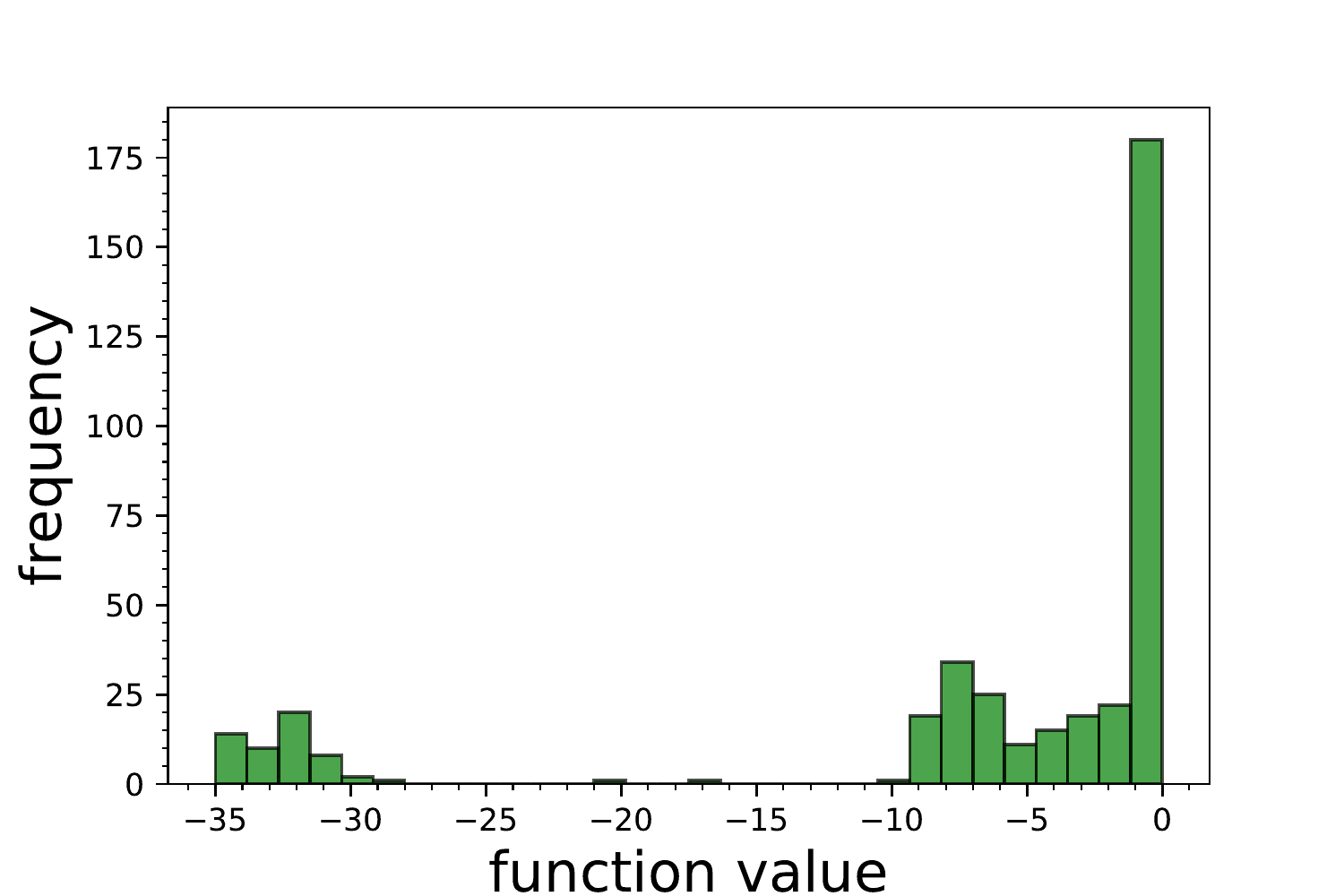}}
    \subfigure[SIPS]{ 
    \includegraphics[width=0.33\textwidth]{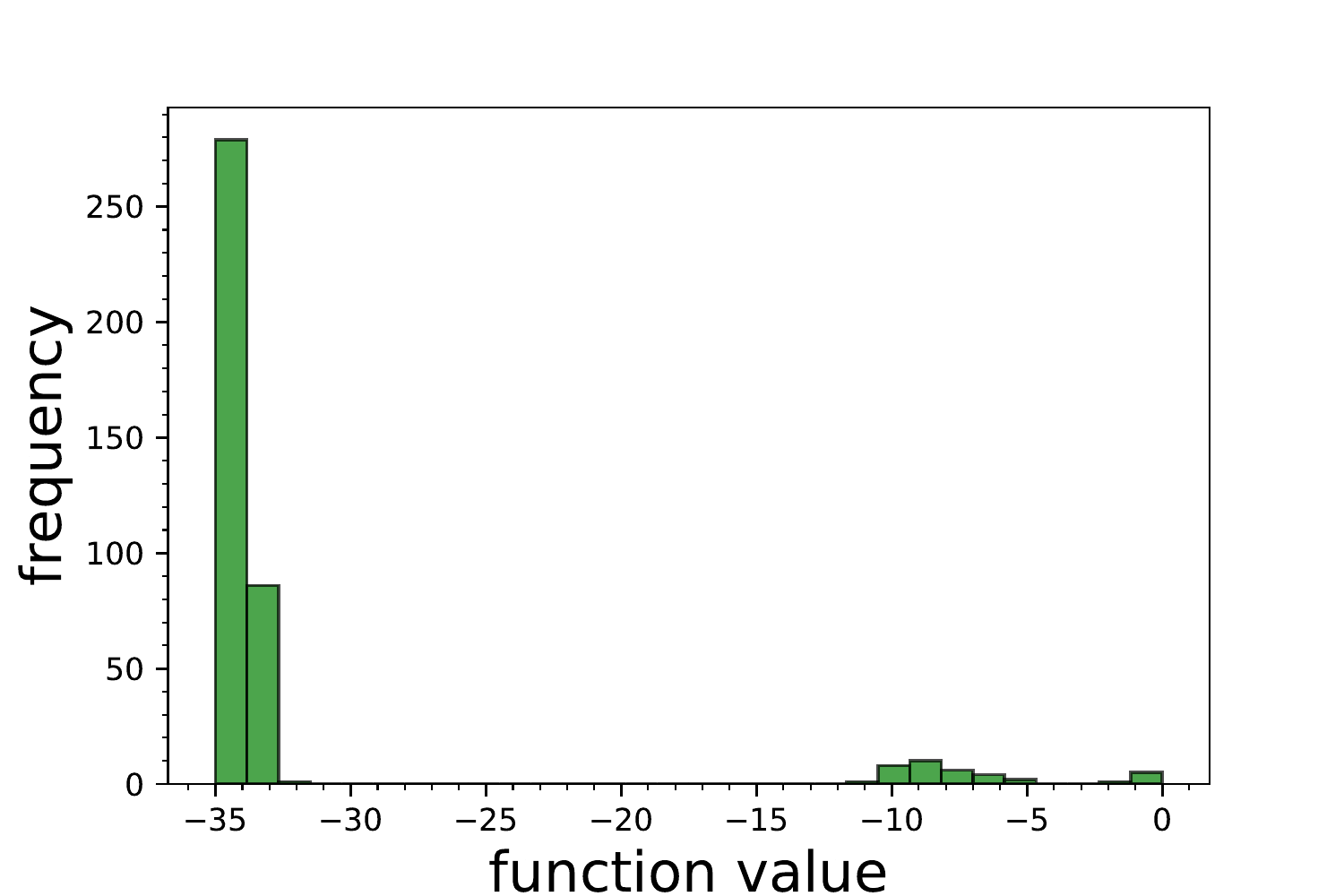}}\\
    \subfigure[OIPS-annealing]{ 
    \includegraphics[width=0.33\textwidth]{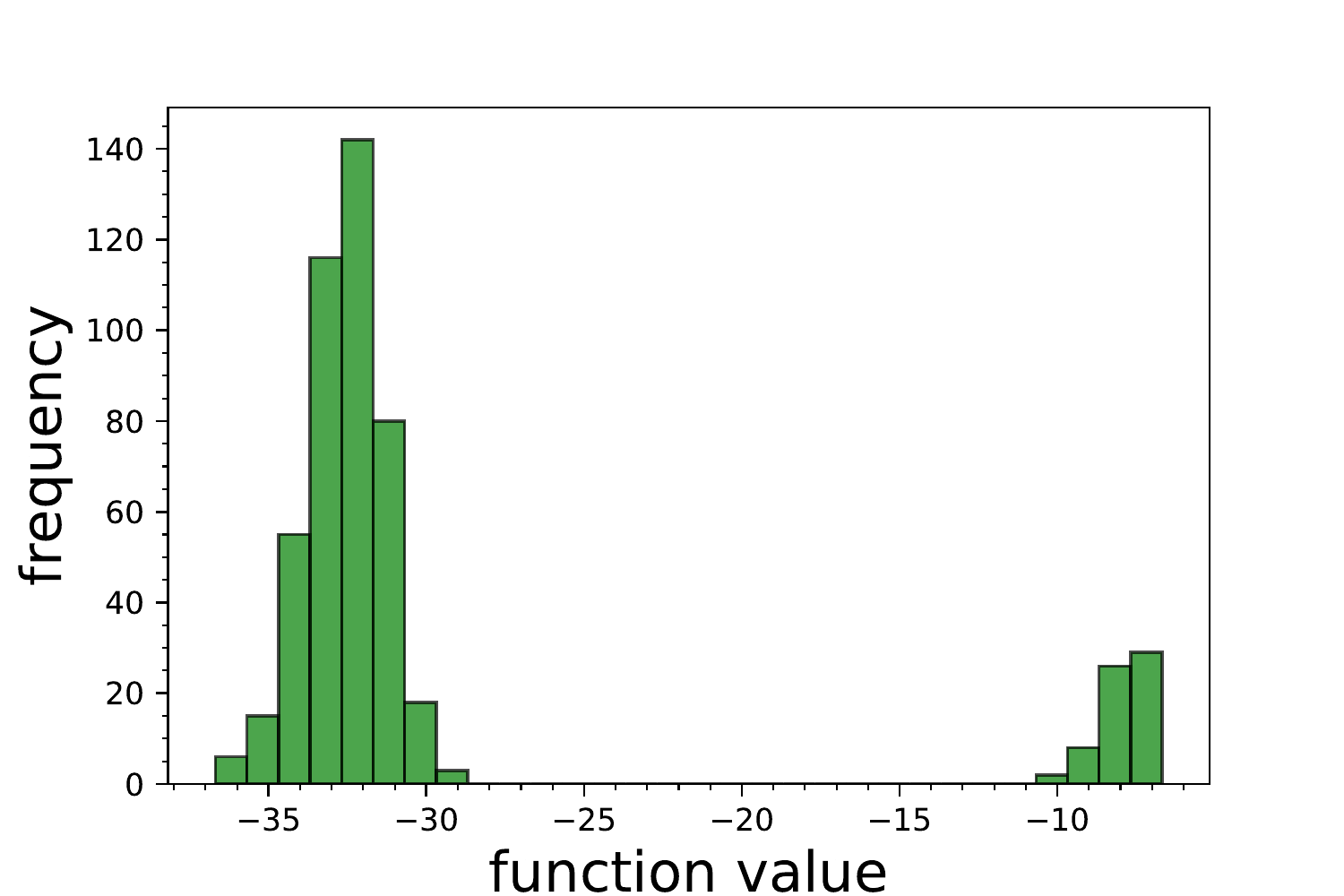}}
    \subfigure[OIPS-SAO]{ 
    \includegraphics[width=0.33\textwidth]{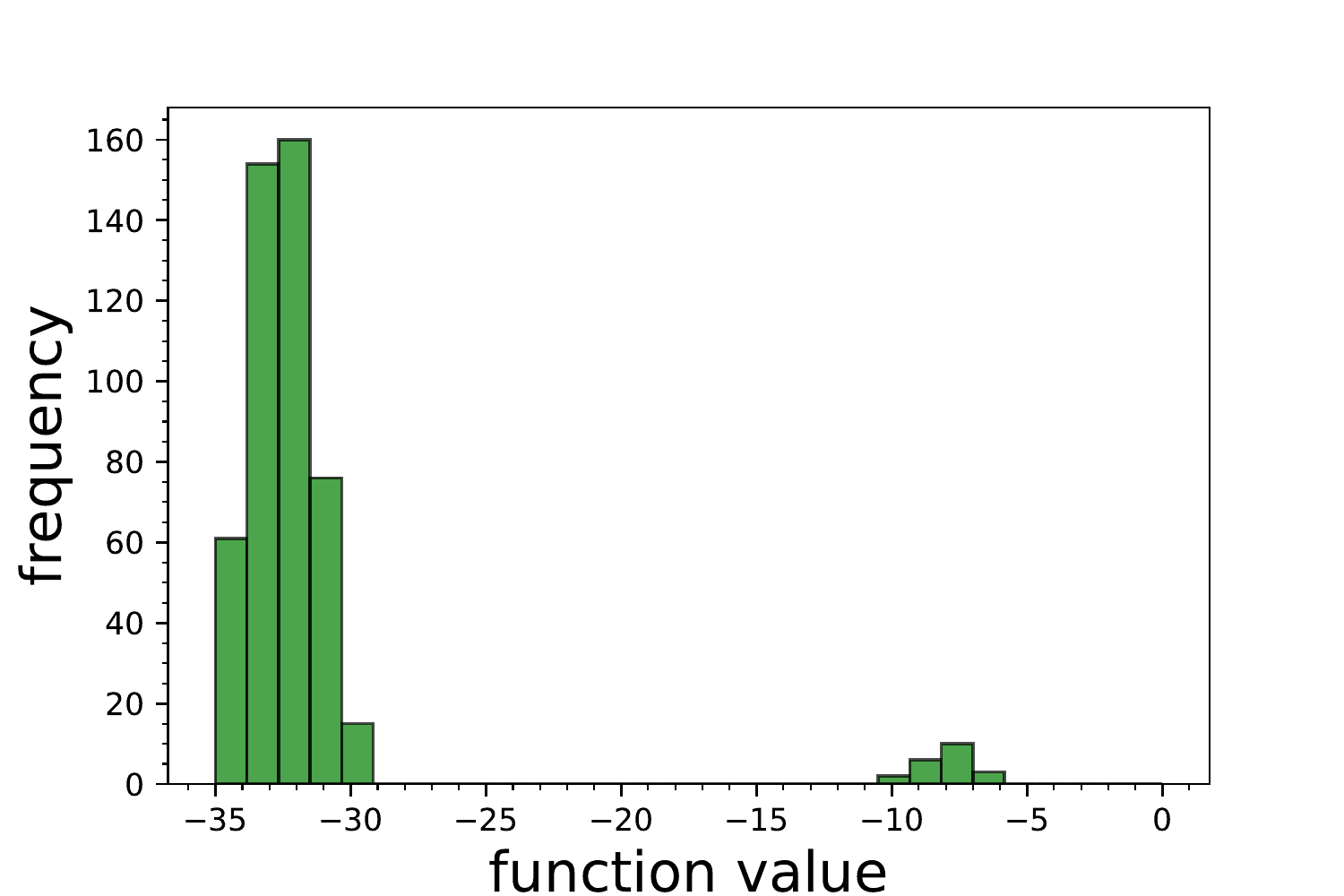}}
    \caption{Histogram of convergent function values of mixture Gaussian density $(d=5, M=10)$.}
    \label{exp2}
\end{figure}

\begin{figure}[ht]
  \subfigure[SIPS]{
        \includegraphics[width=0.33\textwidth]{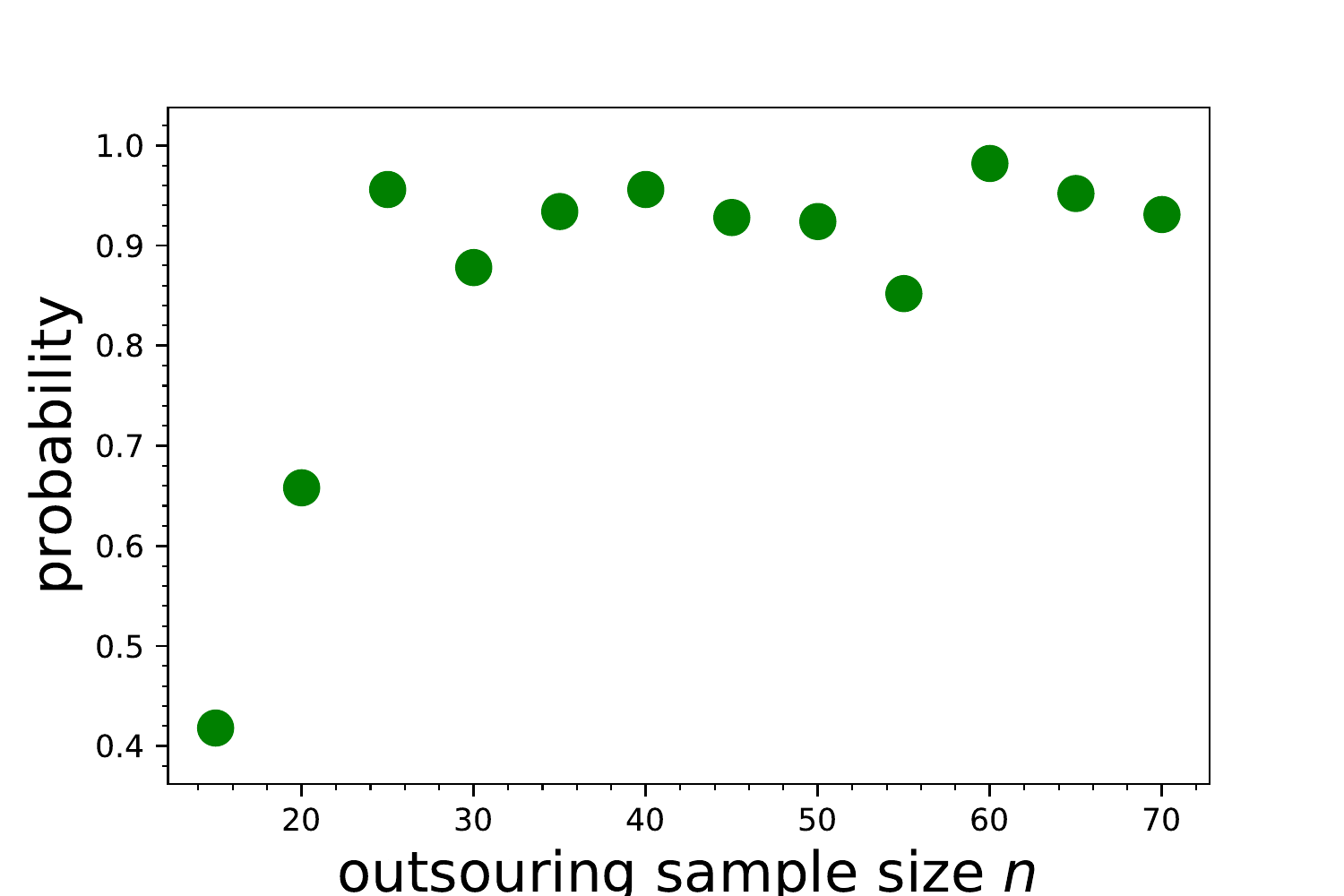}}
  \subfigure[OIPS-annealing]{
          \includegraphics[width=0.33\textwidth]{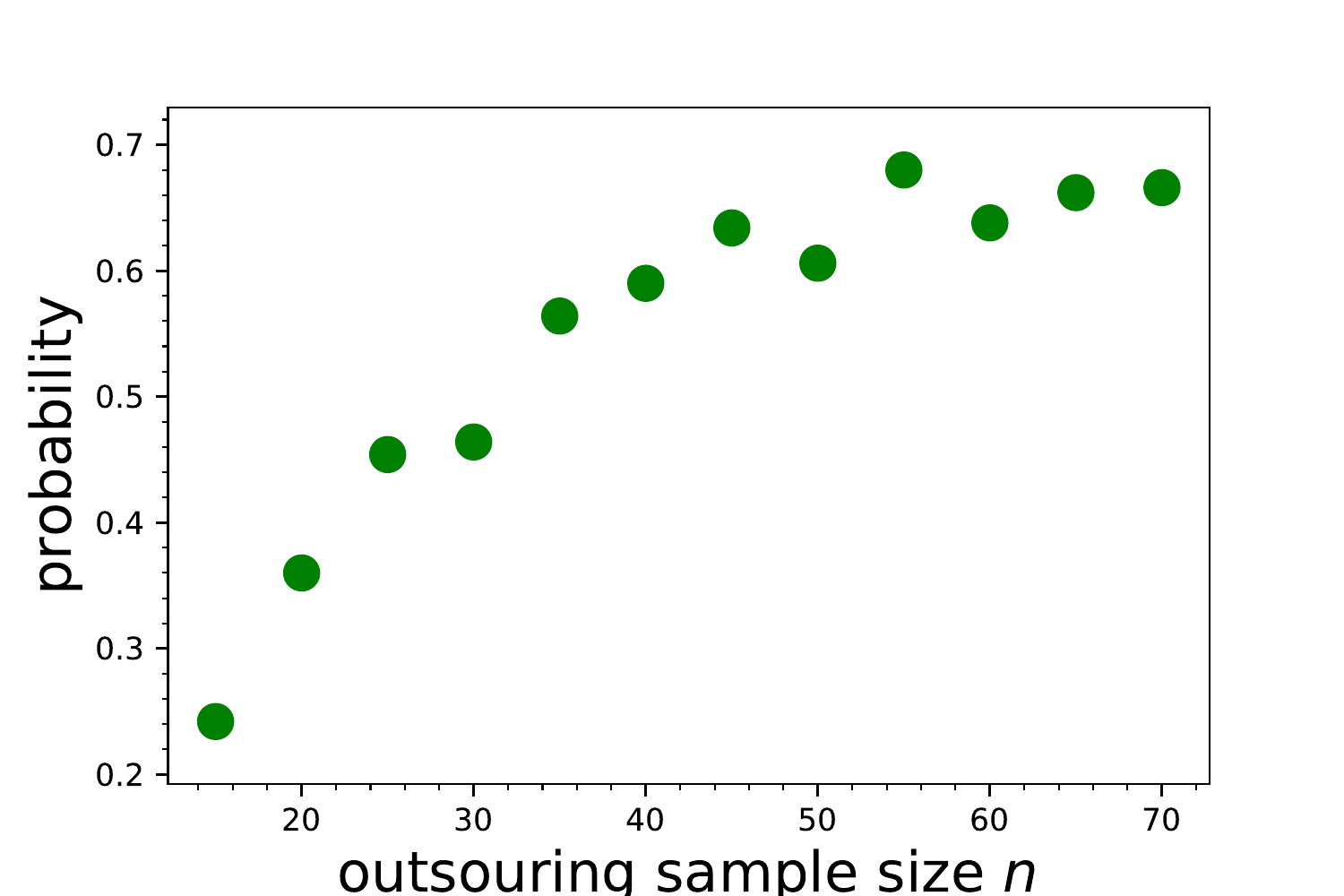}}      
        \caption{Probability of finding the global minimum for different outsourcing sample sizes in OIPS-annealing and SIPS}
         \label{exp2.5}
\end{figure}

\begin{figure}[ht]
\subfigure[Random start]{
    \includegraphics[width=0.33\textwidth]{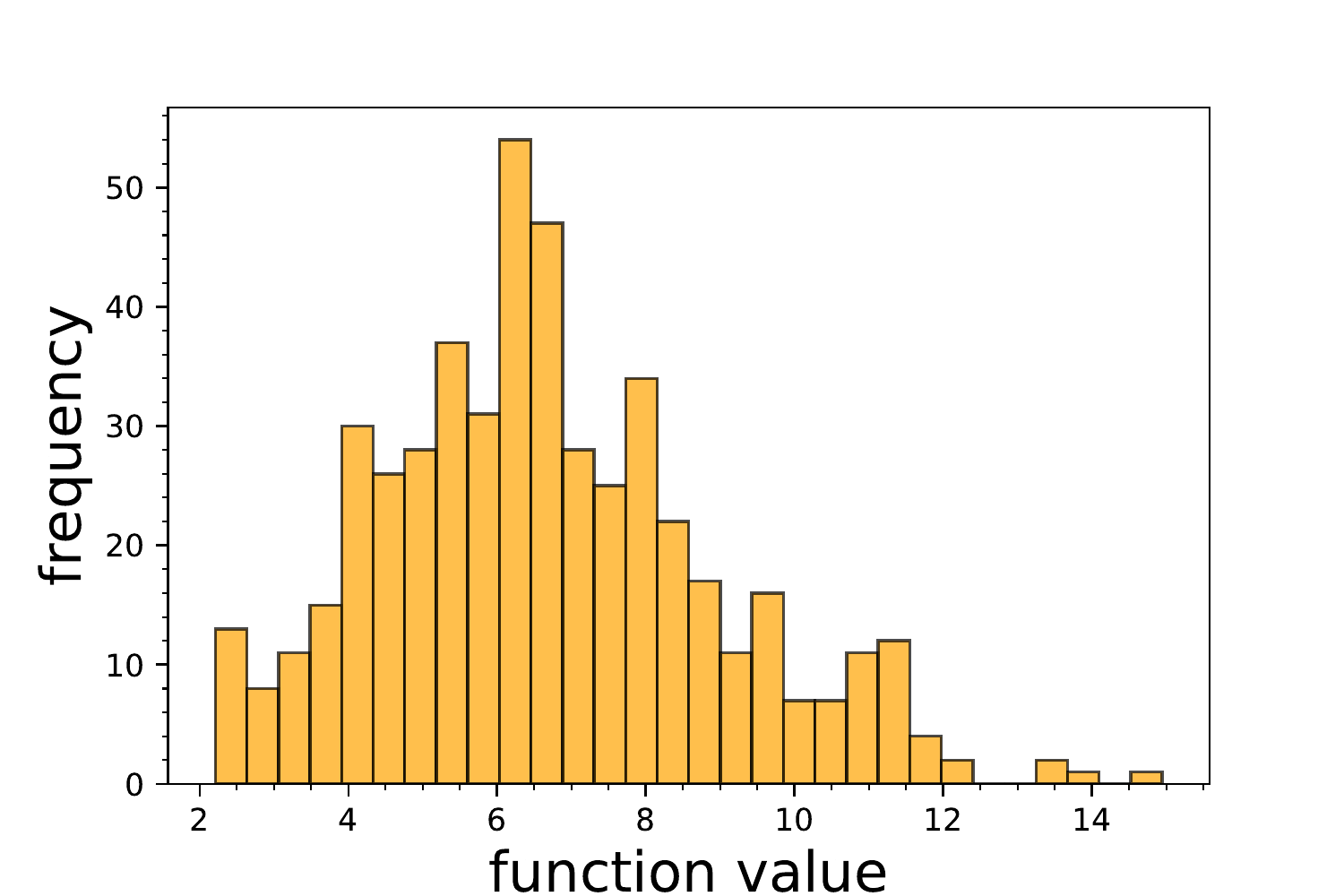}}
    \subfigure[OIPS-annealing]{
    \includegraphics[width=0.33\textwidth]{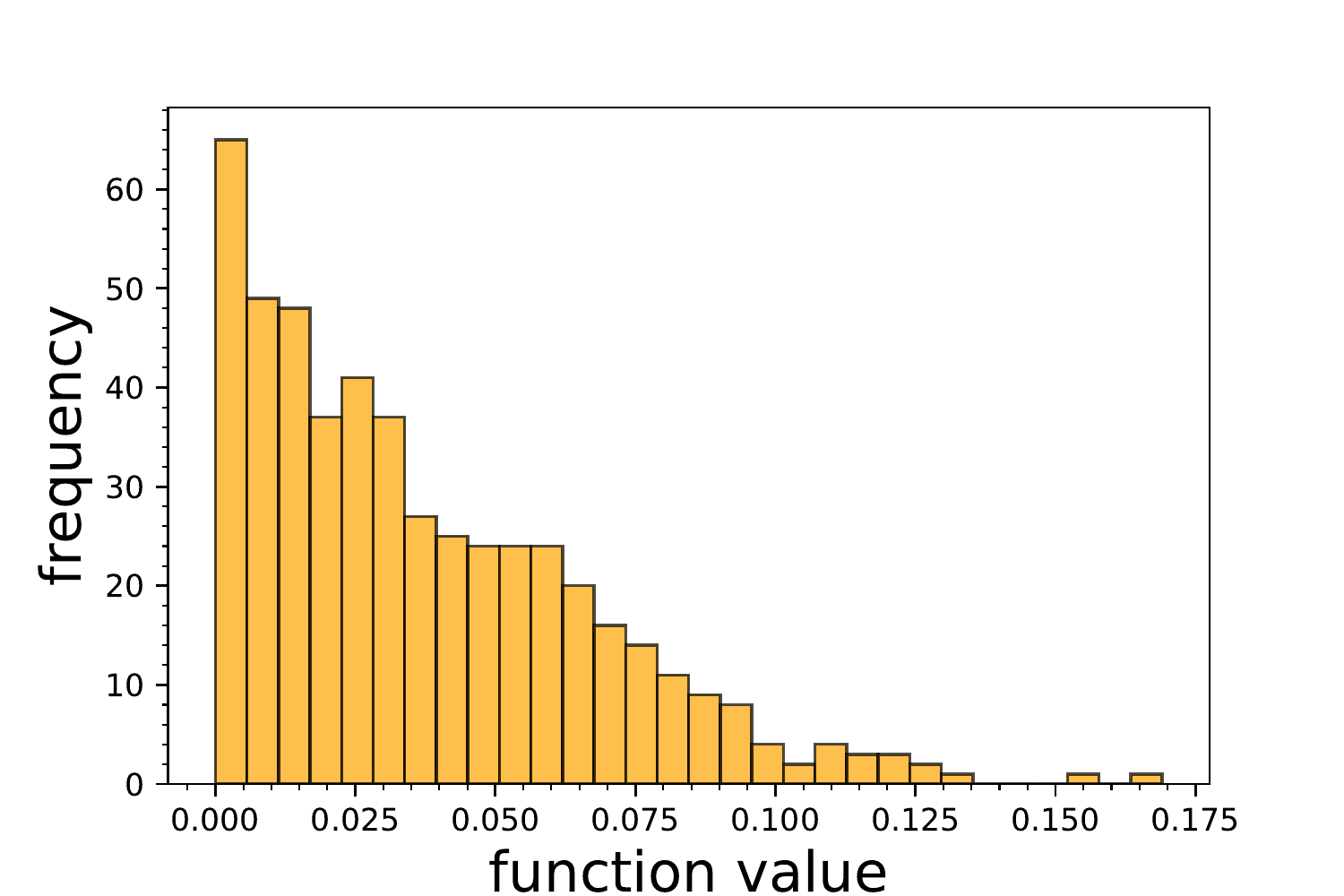}}
    \caption{Histogram of convergent function values of mixture Gaussian density $(d=30, M=20)$.}
    \label{exp2_2}
\end{figure}
    
%
%

\subsection{Generalized multinomial logit model}
We study an application of our algorithms for maximum likelihood estimation of the generalized multinomial logit (GMNL) model. Multinomial logit model is a classic model to study consumer choice. As an extension, the GMNL model accommodates the scaling heterogeneity in utility coefficients through an individual-specific scaling factor \citep{fiebig2010generalized}. Such a generalization makes the negative log-likelihood function nonconvex. In practice, GD or BFGS with random starts are employed for the estimation \citep{train2009discrete}. 

Suppose that there are $N$ customers who make a choice from $J$ alternatives. The utility that customer $n$ chooses alternative $j$ is 
$U_{nj}=x^\top_{j}\phi_n+\epsilon_{nj}$,
where $x_{j}$ is a $p$-dimensional vector of attributes of product $j$, $\phi_n \in \mathbb{R}^p$ is the vector of utility coefficients, and $\epsilon_{nj}$ is an idiosyncratic error term that follows standard Gumbel distribution.  The customer tends to choose products with higher utilities and the probability that $k$ is chosen is 
$P_{nk}={\exp(x^\top_{k}\phi_n)}/ {\sum_{j=1}^J \exp(x^\top_{j}\phi_n)}$. 
The GMNL model specifies $\phi_n$ as $\phi_n=\exp\{ z^\top_n\psi+\xi_n\}\cdot \phi$,
where $z_n$ is a $q$-dimensional vector of agent characteristics, $\psi$ is a $q$-dimensional hetogeneity coefficient, and $\xi_n$ is an independent random shock that follows the standard Gaussian distribution. Let the binary variable $y_{nj} \in \{0,1\} $ denote whether customer $n$ chooses product $j$. 
Then the likelihood of customer $n$'s choice  is  
\[
L_n=\mathbb{E}_{\xi_n \sim \cN(0,1)} \Big[ \prod_{k=1}^J \Big(\frac{\exp(x^\top_{k}\phi_n)}{\sum_{j=1}^J \exp(x^\top_{j}\phi_n)}\Big)^{y_{nk}} \Big].
\]
We use simulation to approximate the above expectation. The model parameter $\theta=(\phi, \psi)$ can be estimated by maximizing the simulated negative log-likelihood function
\begin{align} \label{SLR}
F(\theta)=-\frac{1}{N}\sum_{n=1}^N \log\Big( \frac{1}{R} \sum_{r=1}^R \prod_{k=1}^J \Big(\frac{\exp(x^\top_{k}\phi^{[r]}_n)}{\sum_{j=1}^J \exp(x^\top_{j}\phi^{[r]}_n)}\Big)^{y_{nk}} \Big),
\end{align}
where $\phi^{[r]}_n= \exp\{ z^\top_n\psi+\xi^{[r]}_n\}\cdot \phi  $ is the $r$-th draw from the distribution of $\phi_n$ and $R$ is the total number of draws. 

In our simulation experiment, we consider an instance with dimension parameters $p=10,q=5$, and $J=5$ alternatives. We generate $N=1000$ customers. In particular, we set the true parameter $\phi^*=(1,\ldots,1,-1,\ldots,-1)$ and $\psi^*=(1,\ldots,1)$ and generate product attributes $x_j$ and agent characteristics $z_n$ from standard Gaussian distribution. Then we simulate the agents' choices following the GMNL model and obtain choice data $y_{nj}$.  Based on the simulated dataset $\{x_j,z_n,y_{nj}\}_{1\le n\le N, 1\le j \le J}$, we use gradient descent to optimize the negative log-likelihood \eqref{SLR} with $R=100$. We compare the performance of OIPS-annealing algorithm with random start. For OIPS-annealing, we set the outsourcing sample size $n=200$ and the inverse temperature $\beta=1$. ULA is applied to draw $L=500$ samples as candidate initial points. In the optimization phase, GD is run for $100$ iterations.

Figure \ref{exp4} shows the distribution of convergent objective values (negative log-likelihood). We note that SIPS, OIPS-SAO, and OIPS-annealing again outperform the random start significantly. Moreover, SIPS and OIPS-SAO are performing slightly better than OIPS-annealing.

\begin{figure}[ht]
    \centering      
    \subfigure[Random start]{
    \includegraphics[width=0.33\textwidth]{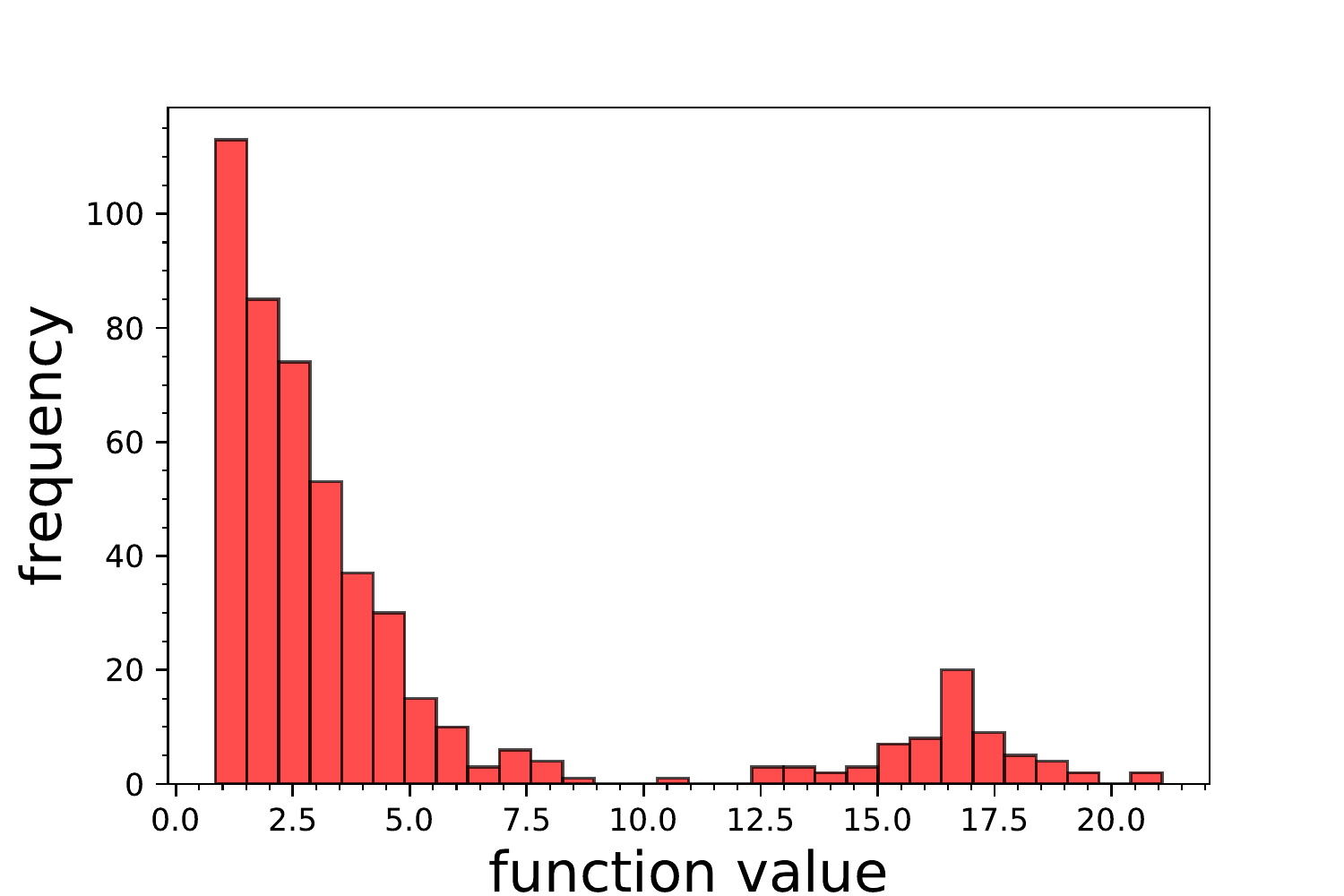}}
    \subfigure[SIPS]{
    \includegraphics[width=0.33\textwidth]{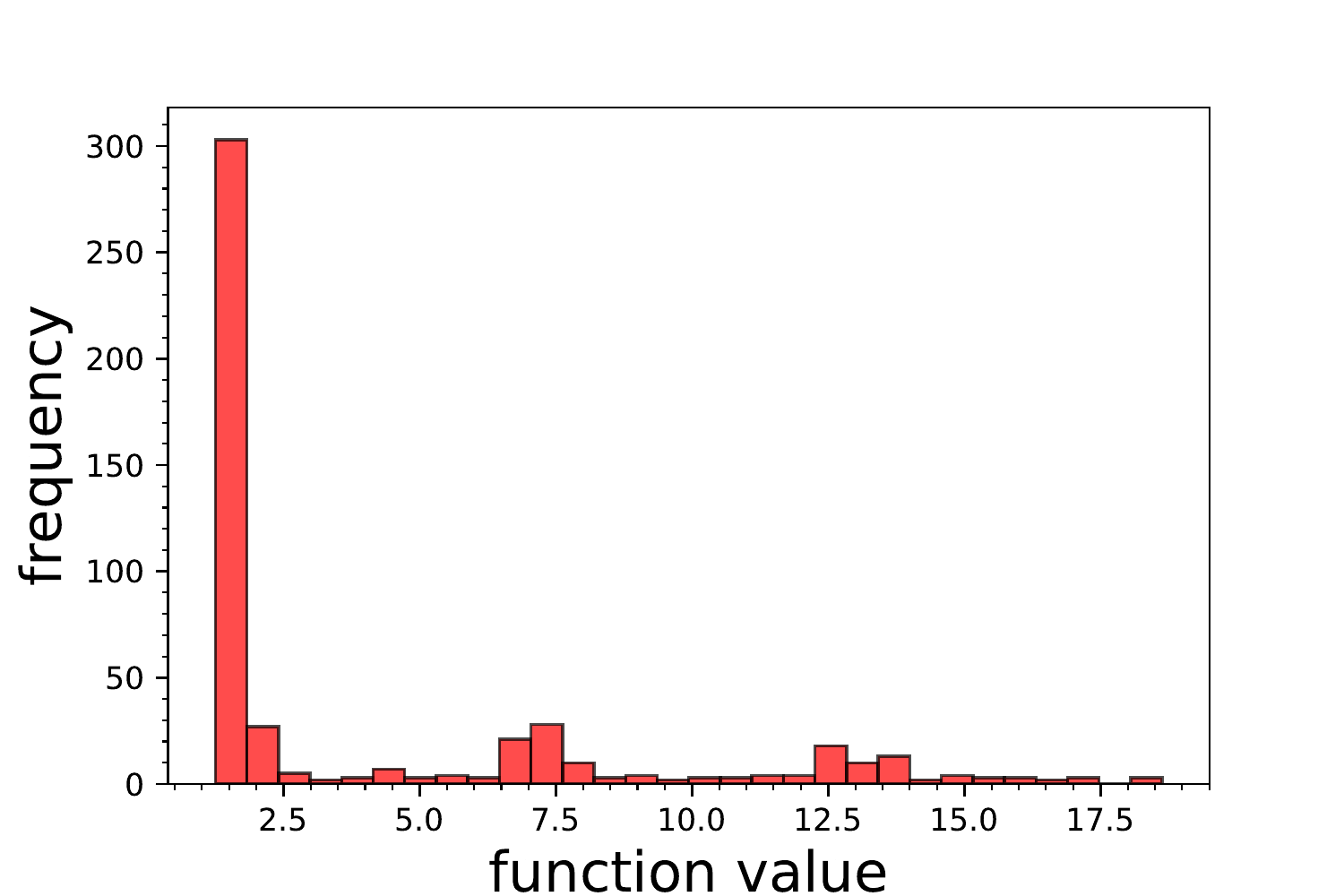}}\\
    \subfigure[OIPS-annealing]{
    \includegraphics[width=0.33\textwidth]{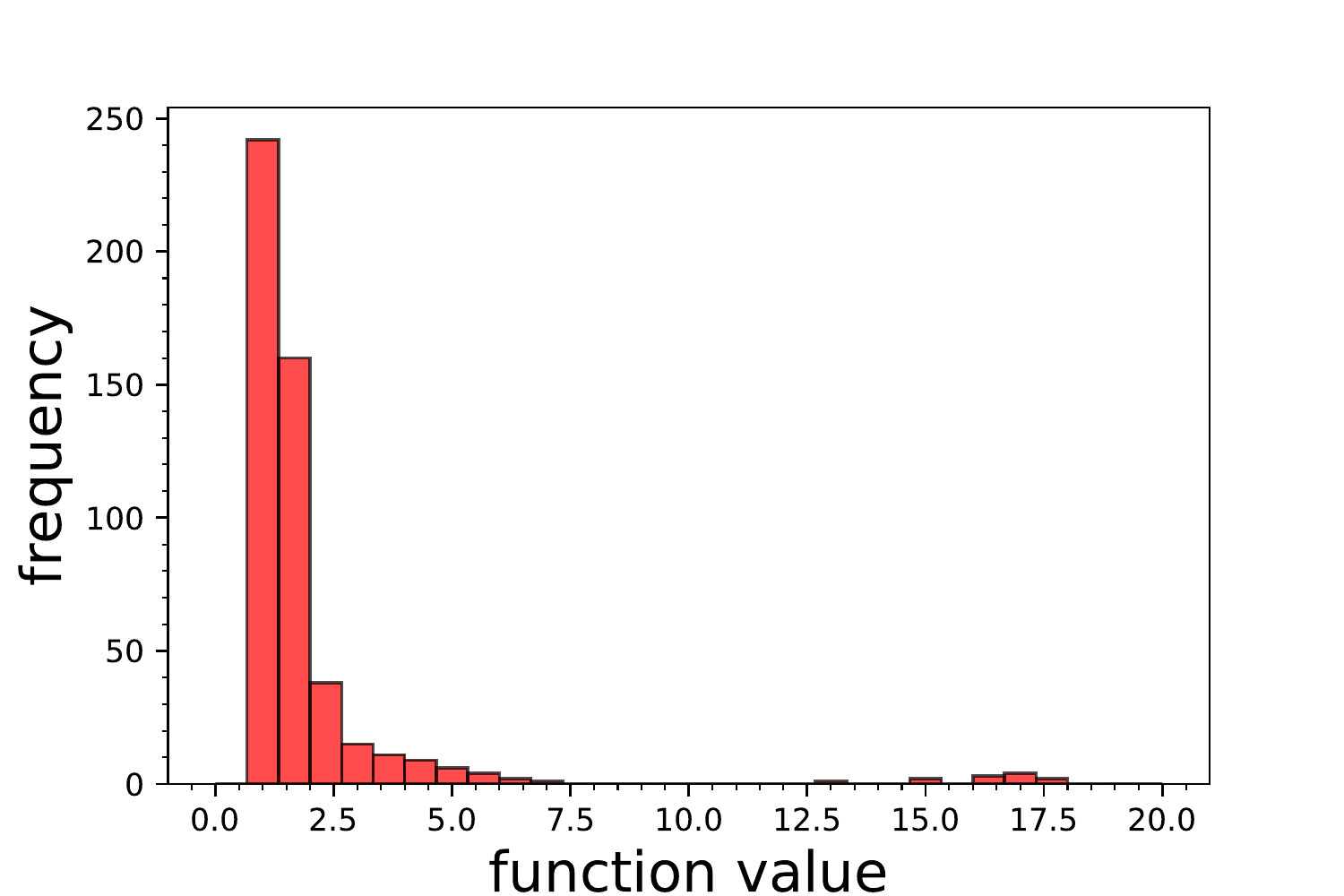}}
    \subfigure[OIPS-SAO]{
    \includegraphics[width=0.33\textwidth]{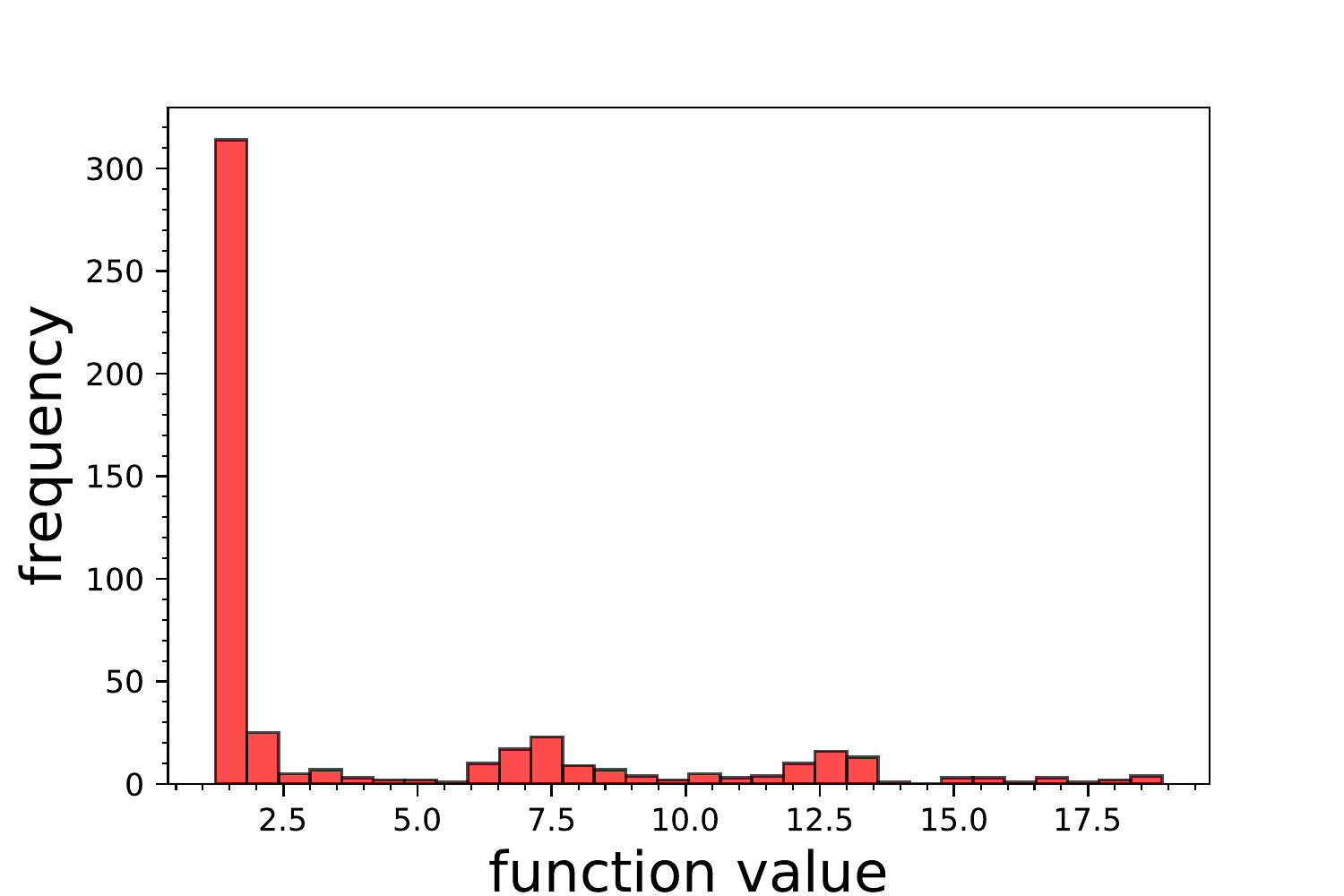}}
    \caption{Histogram of negative log-likelihood of GMNL model}
    \label{exp4}
\end{figure}
\section{Conclusion, Limitations, Future Works} 
We have designed three algorithms using outsourced data to find good initial points. They are better than the popular random start approach. In both theoretical analysis and numerical tests, the SIPS and OIPS-SAO perform better than the OIPS-annealing, but they have computational costs in general. 

Our work has the following two limitations, which can be seen as possible future directions. 1) We assume the outsourced data is drawn randomly from the true population. In practice, such data might be from a biased distribution or need additional privacy encryption.  2) Our analyses focused on the large $\beta$ scenario. In practice, we would prefer to use a moderate $\beta$ due to sampling complexity.


\bibliographystyle{ACM-Reference-Format}
\bibliography{ref}

\appendix

\section{Technical verifications} \label{sec:proof}

\subsection{Approximation accuracy of  $\hat{F}(\theta)$ and data complexity} \label{app:approximation}

\begin{proof}[Proof of Lemma \ref{lem1}]
    To prove the results about gradient and Hessian convergence, we can apply Theorem 1 in \cite{mei2018landscape} directly. Specifically, under Assumptions \ref{aspt:morse} and \ref{ass_c0}, when $n>Cd\log(d)$, with probability at least $1-\rho$, we have 
    \begin{equation}  \label{ineq1}
    \begin{split}
        &\sup_{\theta  \in \Theta} \| \nabla F(\theta)-\nabla \hat{F}_n(\theta) \| \leq \tau \sqrt{\frac{Cd\log(n)}{n}},\\
        &  \sup_{\theta  \in \Theta} \| \nabla^2 F(\theta)-\nabla^2 \hat{F}_n(\theta) \|_{\text{op}} \le \tau^2 \sqrt{\frac{Cd\log(n)}{n}}.
    \end{split}
    \end{equation}    
    
    For the stationary points convergence, based on Theorem 2 in \cite{mei2018landscape}, under Assumptions \ref{aspt:morse} and \ref{ass_c0}, when $n \ge 4Cd\log(n) \cdot ((\tau^2/\sigma^2)\vee(\tau^4/\eta^2))$, the empirical loss function $\hat{F}_n(\theta)$ is $(\sigma/2,\eta/2)$-strongly Morse and possesses $K+1$ stationary points with probability at least $1-\rho$. Furthermore, there is a one-to-one correspondence between $(\theta^*_0,\ldots,\theta^*_K)$, the stationary points of $F(\theta)$, and $(\hat{\theta}^*_0,\ldots,\hat{\theta}^*_K)$, the stationary points of $\hat{F}_n(\theta)$. 
    Moreover, when $n\ge 4Cd\log(n)/\eta^2_*$, 
    \begin{align} \label{ineq2}
    \max_{0\le i \le K}\|\theta_i^*-\hat{\theta}_i^*\|  \le \frac{2\tau}{\eta}\sqrt{\frac{Cd\log(n)}{n}}.
    \end{align}

    It remains to establish the uniform convergence result for $\Fhat_n$. Although it is not directly available in \cite{mei2018landscape}, the proof follows a similar idea. For self-completeness, we provide the details here. 
    
    First of all, given the parameter space $\Theta$, let $\Theta_\varepsilon:=\{\theta_1,\ldots,\theta_J\}$ be a $\varepsilon$-covering net. In other words, for arbitrary $\theta \in \Theta$, there exists certain $ \theta_{j(\theta)} \in  \Theta_\varepsilon$ such that $\|\theta-\theta_{j(\theta)}\|\le \varepsilon$. Thus, for any $\theta \in \Theta$, we have 
    \begin{align} \label{chain_ineq}
    \big|\hat{F}_n(\theta)-F(\theta)\big| \le \big|\hat{F}_n(\theta)-\hat{F}_n(\theta_{j(\theta)})\big|+\big|\hat{F}_n(\theta_{j(\theta)})-F(\theta_{j(\theta)})\big|+\big|{F}(\theta)-F(\theta_{j(\theta)})\big|.
    \end{align}
    For any $t>0$, we denote by 
    \begin{align*}
        A_t&=\Big\{ \sup_{\theta\in \Theta} \big|\hat{F}_n(\theta)-\hat{F}_n(\theta_{j(\theta)})\big|\ge t/3  \Big\}, 
        ~ B_t=\Big\{\sup_{\theta_j \in \Theta_\varepsilon} \big|\hat{F}_n(\theta_j)-{F}(\theta_j)\big|\ge t/3  \Big\},\\
        &\text{and}\ 
        C_t=\Big\{\sup_{\theta\in \Theta} \big|F(\theta)-F(\theta_{j(\theta)})\big|\ge t/3  \Big\}.
    \end{align*}
    Then we have 
    \begin{align*}
        \Pb\Big(\sup_{\theta \in \Theta}\big|\hat{F}_n(\theta)-F(\theta)\big| \ge t\Big)\le \Pb(A_t)+\Pb(B_t)+\Pb(C_t).
    \end{align*}
    In the next, we upper bound the three parts in above inequality respectively. For the last part, we have 
    \[
    \big| F(\theta)-F(\theta_{j(\theta)}) \big| \le \sup_{\theta\in \Theta} \| \nabla F(\theta) \| \cdot \|\theta-\theta_{j(\theta)} \| \le L^* \cdot \varepsilon. 
    \] 
    Hence, when $ t \ge 3\varepsilon L^*$, the deterministic event $C_t$ would never happen and $P(C_t)=0$. For the second part, under Assumption \ref{ass_c0}, by applying the union bound and the sub-Gaussian concentration inequality, we have 
    \begin{align*}
    \Pb(B_t) &\le \big | \Theta_\varepsilon \big|\cdot P\Big(\big|\hat{F}_n(\theta_j)-{F}(\theta_j)\big|\ge t/3\Big) \\
    & \le | \Theta_\varepsilon |\cdot \exp\big\{-nt^2/(18\tau^2)\big\} \le (2/\varepsilon)^d\cdot \exp\big\{-nt^2/(18\tau^2)\big\}.
    \end{align*}
    Thus, when 
    \[
    t> 5\tau\cdot \sqrt{\frac{\log(2/\rho)+d\log(2/\varepsilon)}{n}},
    \]  
    we have 
    $  \Pb(B_t)\le \rho/2. $
    For the first part, by Markov inequality, we have 
    \[
        \Pb(A_t) \le \frac{3\E\big[\sup_{\theta\in \Theta} |\hat{F}_n(\theta)-\hat{F}_n(\theta_{j(\theta)})|\big]}{t} \le \frac{3\varepsilon\cdot \E\big[\sup_{\theta\in \Theta} \|\nabla \hat{F}_n(\theta)\|\big]}{t}.
    \]
    By Assumption \ref{aspt:morse}, we have
    \begin{align*}
        &\E\Big[\sup_{\theta\in \Theta} \|\nabla \hat{F}_n(\theta)\|\Big] \le  \E\Big[\sup_{\theta\in \Theta} \|\nabla \hat{F}_n(\theta)-\nabla \hat{F}_n(\theta^*)\|\Big] + \E\big[ \| \nabla \hat{F}_n(\theta^*)\|\big] \le 2J^*+H, 
    \end{align*}
    which implies that 
    \[
    \Pb(A_t) \le  3\varepsilon(2J^*+H)/t.
    \]
   Taking $t\ge 6\varepsilon(2J^*+H)/\rho$, we have 
    $ P(A_t) \le \rho/2$. 
    
    Finally, by taking 
    \[\varepsilon^*=\rho\tau/(6dn(2J^*+H)), ~~ t^*=5\tau\sqrt{(\log(2/\rho)+d\log(2/\varepsilon))/n},\] 
    and utilizing the fact that $H \le \tau^2 d^{c_h}, J_*\le \tau^3 d^{c_h}$, when $n \ge Cd\log(d)$, we have 
    \[\Pb\left(\sup_{\theta \in \Theta}\big|\hat{F}_n(\theta)-F(\theta)\big|\ge \tau\sqrt{\frac{Cd\log(n)}{n}}\right)\le \rho.\] 
    
    Now, given an approximation accuracy $\delta$, we calculate the minimal required sample size. 
    For arbitrary positive constant $\iota$, there exists an absolute constant $C_\iota$ such that  $\log(n)\le C_\iota \cdot n^\iota$. As a result, when 
    $$
    n \ge \max\Big\{\Big[ \frac{C_\iota Cd}{(\delta/((2\tau/\eta)\vee\tau \vee \tau^2))^2} \Big]^{\frac{1}{1-\iota}}, 4Cd \big(\log(d)\vee \log(n)/\eta^2_*\big)\Big\}, 
    $$
    we have 
    \begin{align*}
        &\sup_{\theta  \in \Theta} |F(\theta)-\hat{F}_n(\theta)|\leq \delta,\  \sup_{\theta  \in \Theta} \| \nabla F(\theta)-\nabla \hat{F}_n(\theta) \|\leq \delta,\\
        &\sup_{\theta  \in \Theta} \| \nabla^2 F(\theta)-\nabla^2 \hat{F}_n(\theta) \|_{\text{op}} \le \delta,\ \mbox{ and }  \max_{0\le i \le K}\|\theta_i^*-\hat{\theta}_i^*\|  \le \delta.
   \end{align*}
   with probability at least $1-\rho$.  
\end{proof}

\subsection{Performance analysis of the sampling approach}

\begin{proof} [Proof of Proposition \ref{lem2}] 
    First note that when $\hat{F}_n(\theta)$ is a $\delta$-approximation, for any $\theta \notin \cB_r(\theta^*_0)$, by Assumption \ref{ass1} we have 
    \begin{align*}
       \hat{F}_n(\theta)-\hat{F}_n({\theta}^*_0)  \ge  \big( {F}(\theta)-\delta \big) -  \big({F}({\theta}^*_0)+\delta \big)  \ge \alpha-2\delta >0.
    \end{align*} Hence, by the definition of $\pi_\beta$, we have 
    \begin{align*}
    \Pb\Big(\tilde\theta_{\beta}\in\cB_r(\theta^*_0)\Big)& =\frac{\int_{\cB_r(\theta^*_0)} \exp(-\beta\hat{F}_n(\theta))d\theta}{\int_{\cB_r(\theta^*_0)} \exp(-\beta\hat{F}_n(\theta))d\theta+\int_{\Theta/\cB_r(\theta^*_0)} \exp(-\beta\hat{F}_n(\theta))d\theta} \\
    &=\frac{\int_{\cB_r(\theta^*_0)} \exp(-\beta[\hat{F}_n(\theta)-\hat{F}_n({\theta}^*_0)])d\theta}{\int_{\cB_r(\theta^*_0)} \exp(-\beta[\hat{F}_n(\theta)-\hat{F}_n({\theta}^*_0)])d\theta+\int_{\Theta/\cB_r(\theta^*_0)} \exp(-\beta[\hat{F}_n(\theta)-\hat{F}_n({\theta}^*_0)])d\theta}\\
    &\ge \frac{\int_{\cB_r(\theta^*_0)} \exp(-\beta[\hat{F}_n(\theta)-\hat{F}_n({\theta}^*_0)])d\theta}{\int_{\cB_r(\theta^*_0)} \exp(-\beta[\hat{F}_n(\theta)-\hat{F}_n({\theta}^*_0)])d\theta+\exp(-\beta(\alpha-2\delta))\cdot\text{Vol}(\Theta/\cB_r(\theta^*_0))},
    \end{align*}
    where $ \text{Vol}(\Theta/\cB_r(\theta^*_0))$ denotes the volume of set $\Theta/\cB_r(\theta^*_0)$. 
    
    On the other hand, based on the regularity condition of Hessian  and the definition of $\delta$-approximation, we have $ \|\nabla^2\hat{F}_n(\theta)\|_{\text{op}} \le H+L^*+\delta$. As a result, for any $\theta \in \cB_r(\theta^*_0)$, 
    \begin{align*} 
    \hat{F}_n(\theta)-\hat{F}_n({\theta}^*_0) \le 2(H+L^*+\delta)\cdot  \big(  \|\theta-{\theta}^*_0\|^2\big).
    \end{align*}
    Hence, 
    \begin{align*}
        & \int_{\cB_r(\theta^*_0)} \exp\big(-\beta[\hat{F}_n(\theta)-\hat{F}_n({\theta}^*_0)] \big)d\theta 
        \ge   \int_{\cB_r(\theta^*_0)} \exp\big(-2\beta(H+L^*+\delta)\|\theta-\theta^*_0 \|^2\big)d\theta \\
        =& \Big( {\pi\beta^{-1}}/{(H+L^*+\delta)}\cdot \big(\Psi(2r\sqrt{\beta(H+L^*+\delta))})-1/2\big) \Big)^d, 
    \end{align*}
    where $\Psi(\cdot)$ denotes the CDF of standard normal distribution. Note that when $\beta \ge \Omega(r^{-2})$, $$ \Psi(2r\sqrt{\beta(H+L^*+\delta))})-1/2=O(1).$$   
    Then,  
    \[
    \int_{\cB_r(\theta^*_0)} \exp\big(-\beta[\hat{F}_n(\theta)-\hat{F}_n(\hat{\theta}^*_0)] \big)d\theta=O \Big( \big(\beta^{-1}/(H+L^*+\delta)\big)^d \Big).
    \] 
    As a result,   
    \[
       1/ \Pb\big(\tilde\theta_{\beta}\in\cB_r(\theta^*_0)\big) = O\Big({1+\exp\big\{ -\beta\big(\alpha-2\delta \big) \big \}\big/\big(\beta^{-1}/(H+L^*+\delta)\big)^d}\Big),
     \]
    which further implies that
    \[
    1-\Pb\Big(\tilde\theta_{\beta}\in\cB_r(\theta^*_0)\Big)=O\Big( \exp\big\{ -\beta\big(\alpha-2\delta \big) \big \}\big/ \big(\beta^{-1}/(H+L^*+\delta)\big)^d  \Big).
    \]
    Finally, by setting $\delta=\alpha/4$, we obtain the result.  
\end{proof}

%
%
    \begin{proof}[Proof of Lemma \ref{approximation_sampling}]
    Let $X_1,\ldots,X_L$ be samples from $\hat{\cM}$.     
    \begin{align*}
    \Pb(X_1\notin B,\ldots, X_L\notin B)
    &=\E \Big[ \prod^L_{i=1}1_{(X_i\notin B)}\Big]\\
    &=\E \Big[ \prod_{i=1}^{L-1}1_{(X_i\notin B)}\cdot \E_{L-1}\big[1_{X_L\notin B}\big]\Big]\\
    &=\E \Big[\prod_{i=1}^{L-1}1_{(X_i\notin B)}\cdot \pihat_{X_{L-1}} (B^c)\Big]\\
    &\leq \E \Big[\prod_{i=1}^{L-1}1_{(X_i\notin B)}\cdot  )(\pi_\beta(B^c)+\delta_\beta) \Big]\\
    &=(\pi_\beta(B^c)+\delta_\beta)\cdot \Pb(X_1\notin B,\ldots, X_{L-1}\notin B).
    \end{align*}
    By induction, we have
    \[
    \Pb(X_1\notin B,\ldots, X_L\notin B)\leq (\pi_\beta(B^c)+\delta_\beta)^L.
    \]
    \end{proof}
%
\begin{proof}[Proof of Theorem \ref{thm0}]
We use $\cI_n(\delta)$ to denote the random event that $\hat{F}_n(\theta)$ is a $\delta$-approximation of $F(\theta)$. First, based on Proposition \ref{lem2}, $\Pb(\cI^c_n(\delta))\leq \rho$ for $n\ge n(\delta,\rho,d)$. Then, 
\[
\Pb(\cF_0)-\rho\leq \Pb\big(\cF_0\cap I_n(\delta)\big)\leq \Pb\big(\cF_0|I_n(\delta)\big).
\]
By the definition of $\delta$-approximation, conditional on $\cI(\delta)$,  if at least one of $(\theta_1,\cdots,\theta_L)$ falls into $\cB_r(\theta^*_0)$, $\cF_0$ would not happen. Hence, by Lemma \ref{approximation_sampling}, we have 
$$  
\log\big(\Pb(\cF_0|\cI_n(\delta))\big)\leq  L\cdot \log \big(\pi_\beta(\cB^c_r(\theta^*_0))+\delta_\beta\big)
\leq L\cdot \max\big\{\log \big(2\pi_\beta(\cB^c_r(\theta^*_0)), \log(2\delta_\beta)\big\}.
$$
Finally, based on Lemma \ref{lem2}, when $\beta \ge \Omega(r^{-2})$,
\begin{align*}
    \log \big(\pi_\beta(\cB^c_r(\theta^*_0))\big)=O\big(-\beta\big(\alpha-2\delta \big) -d \log \beta\big).
\end{align*}
If we set $\delta=\alpha/4$, the above upper bound leads to  
\[
\Pb(\cF_0)\leq \rho+\exp \big(-CL\cdot \max\big\{-\beta\alpha/2+d\log(\beta),\log(2\delta_\beta)\big\}\big).
\]
for some constant $C>0$, and we finish the proof. 
\end{proof}

\subsection{Performance analysis of the optimization approaches} 
In this section, we establish the performance guarantee of the optimization approach, i.e., Algorithm \ref{alg:optimization}. 
We first analyze the sample selection rule with the SAO approach, which set $\theta_{*}^0=\hat\cT(\theta_{i^*}^0)$
where $i^*=\text{argmin}_{1\leq i\leq L} \big\{  \hat{F}_n(\hat{\cT}(\theta_{i}^0))\big\}.$ 
\begin{proof} [Proof of Theorem \ref{thm1}] 
Under Assumption \ref{ass_ab}, $F(\theta)$ is $\mu$-strongly convex in $\cB_r(\theta^*_0)$. When $\delta<\mu$ and $\hat{F}_n(\theta)$ is a $\delta$-approximation, we have 
\[
\sup_{\theta \in \Theta} \|\nabla^2 F(\theta)- \nabla^2 \hat{F}_n(\theta)\|_{\text{op}} \le \delta \mbox{ and } \|\hat{\theta}^*_0-\theta^*_0\| \le \delta.
\]
This implies that $ \hat{F}_n(\theta)$ is $(\mu-\delta)$-strongly convex in $\cB_r(\theta^*_0)$ and $\hat{\theta}^*_0$ is the unique minimum of  $\hat{F}_n(\theta)$ in $\cB_r(\theta^*_0)$. Hence, starting from any $\theta \in \cB_r(\theta^*_0)$, the optimization algorithm $\hat{\cT}$ can converge to $\hat{\theta}^*_0$, which implies that $$\mathbb{P}(\hat{\cT}(\tilde{\theta}_\beta)\ne \hat{\theta}^*_0) \le \mathbb{P}(\tilde{\theta}_\beta \notin \cB_r(\theta^*_0)).$$ Then by Proposition \ref{lem2}, we can establish the upper bound for the probability of $\hat{\cT}(\tilde{\theta}_\beta)\ne \hat{\theta}^*_0$. Finally, note that when at least one of $(\theta_1,\theta_2,\ldots,\theta_L)$ falls into $\cB_r(\theta^*_0)$, the global minimum of $\hat{F}_n(\theta)$, $\hat{\theta}^*_0 \in \cB_r(\theta^*_0)$, can be found by $\hat{\cT}$ and would also be selected as the initial point to optimize $F(\theta)$. As a result, $\cF_1\subset \cF_0$. Thus, by Theorem \ref{thm0}, we prove the upper bound for $ \PP(\cF_1)$.  
\end{proof}

\begin{proof}[Proof of Theorem \ref{thm2}]
    We first show that if at least one point of $(\theta_1,\ldots,\theta_L)$ is in $\cB_{r_0}(\theta^*_0)$, the annealing approach would select a point that falls into $\cB_{r_0}(\theta^*_0)$. When $\hat{F}_n(\theta)$ is a $\delta$-approximation of $F(\theta)$, if $\theta_i \in \cB_{r_0}(\theta^*_0)$, we have 
    \begin{align*}
        \hat{F}_n(\theta_i) \le F(\theta_i) +\delta & \le  F(\theta^*_0) +\delta+\frac{1}{2}r_0^2 \cdot  \sup_{\theta \in \Theta}\|\nabla^2 F(\theta)  \|_{\text{op}} 
        \\& \le  F(\theta^*_0) +\delta+\frac{\alpha}{2} \\
        & \le \min_{j\ne i} F(\theta_j)-\frac{\alpha}{2} +\delta \mbox{ by Assumption \ref{ass1}}
        \\ & \le \min_{j\ne i} \hat{F}_n(\theta_j) \mbox{ as $\delta \le \alpha/4$}. 
    \end{align*}
Hence, if the algorithm selects some $\theta_j \ne \theta_i $, we must have $\hat{F}_n(\theta_j) \le \hat{F}_n(\theta_i)$, which implies that $ \theta_j$ is in $\cB_{r_0}(\theta^*_0)$ as well. The remaining proof follows exactly the same line of argument as that of Theorem \ref{thm0}.
\end{proof}

\subsection{Performance analysis for extension to $\epsilon$-Global Minimum}
\begin{proof}[Proof of Theorem \ref{thm3}]
For Algorithm \ref{alg:optimization}-annealing, note that if there is a sample $\theta_i \in \cB_{r_\epsilon}(\theta^*_i) \subseteq \cB_{\epsilon,r_\epsilon}$, then we have 
\begin{align*}
\hat{F}_n(\theta_i) & \le F(\theta_i)+\delta \\
& \le F(\theta^*_i)+ \sup_{\theta \in \Theta} \| \nabla^2 F(\theta)\|_{\text{op}}\cdot r_\epsilon^2 +\delta \le  F(\theta^*_0)+ 2 \epsilon +\delta.
\end{align*}
So if we select some $\theta_j$ instead, then
$$
F(\theta_j) \le \hat{F}_n(\theta_j)+\delta \le \hat{F}_n(\theta_i)+\delta \le F(\theta^*_0)+ 2 \epsilon +2\delta.
$$
When $\delta=\epsilon/4$, by definition $\theta_j$ is a $3\epsilon$-global minimum. 

In the next, we estimate the probability that a sample drawn from $\pi_\beta$ falls into $\cB_{\epsilon,r_\epsilon}$. Note that for $\theta \notin \cB_{\epsilon,r_\epsilon}$ ,
$$
\hat{F}_n(\theta)-\hat{F}_n(\theta^*_0) \ge \epsilon-2\delta.
$$
Then we have 
\begin{align*}
\mathbb{P}\big( \tilde{\theta}_\beta \in \cB_{\epsilon,r_\epsilon}  \big)& =\frac{\int_{\cB_{\epsilon,r_\epsilon}} \exp(-\beta\hat{F}_n(\theta))d\theta}{\int_{\cB_{\epsilon,r_\epsilon}} \exp(-\beta\hat{F}_n(\theta))d\theta+\int_{\Theta/\cB_{\epsilon,r_\epsilon}} \exp(-\beta\hat{F}_n(\theta))d\theta} \\
&\ge \frac{\int_{\cB_{\epsilon,r_\epsilon}} \exp(-\beta[\hat{F}_n(\theta)-\hat{F}_n({\theta}^*_0)] )d\theta}{\int_{\cB_{\epsilon,r_\epsilon}} \exp(-\beta[\hat{F}_n(\theta)-\hat{F}_n({\theta}^*_0)])d\theta+\exp(-\beta(\epsilon-2\delta))\cdot\text{Vol}(\Theta/\cB_{\epsilon,r_\epsilon})} \\
& \ge \frac{\int_{\cB_{r_\epsilon}(\theta^*_0)} \exp(-\beta[\hat{F}_n(\theta)-\hat{F}_n({\theta}^*_0)] )d\theta}{\int_{\cB_{r_\epsilon}(\theta^*_0)} \exp(-\beta[\hat{F}_n(\theta)-\hat{F}_n({\theta}^*_0)])d\theta+\exp(-\beta(\epsilon-2\delta))\cdot\text{Vol}(\Theta/\cB_{\epsilon,r_\epsilon})}.
\end{align*}
Similar to the proof of Proposition \ref{lem2}, when $\beta\ge \Omega(r_\epsilon^{-2})$, we have    
\begin{align*}
    \int_{\cB_{r_\epsilon}(\theta^*_0)} \exp\big(-\beta[\hat{F}_n(\theta)-\hat{F}_n({\theta}^*_0)]\big) d\theta
    = O \Big(\big(\beta^{-1}/(H+L^*+\delta)\big)^d \Big).
\end{align*}
As a result, 
$$
1/\mathbb{P}\big( \tilde{\theta}_\beta \in \cB_{\epsilon,r_\epsilon}  \big) = O\Big({1+ \exp\big\{ -\beta\big(\epsilon-2\delta \big) \big \}\big/\big(\beta^{-1}/(H+L^*+\delta)\big)^d}\Big),
$$
which further implies that
\[
\Pb\Big(\tilde\theta_{\beta}\notin\cB_{\epsilon,r_\epsilon} \Big)=O\Big(1+\exp\big\{ -\beta\big(\epsilon-2\delta \big) \big \}\big/\big(\beta^{-1}/(H+L^*+\delta)\big)^d  \Big).
\]
When $\delta=\epsilon/4$, we have 
\[
\log\big( \Pb(\tilde\theta_{\beta}\notin\cB_{\epsilon,r_\epsilon})\big)= O\big( -\beta\epsilon/2+d\log (\beta) \big).
\]
Hence, with probability at least $O(\exp\{-\beta\epsilon/2+d\log (\beta)\} ) $, Algorithm \ref{alg:optimization} with subroutine 1 can find a $3\epsilon$-global minimum of $F(\theta)$.   

We use $\cI_n(\delta)$ to denote the random event that $\hat{F}_n(\theta)$ is a $\delta$-approximation of $F(\theta)$. Similar to the proof of Theorem \ref{thm0}, we have 
\begin{align*}
  \Pb(\cF_{3\epsilon,2})-\rho\leq \Pb\big(\cF_{3\epsilon,2}\cap I_n(\delta)\big)\leq \Pb\big(\cF_{3\epsilon,2}|I_n(\delta)\big).
\end{align*}
and 
\begin{align*}  
  \log\big(\Pb(\cF_{3\epsilon,2}|\cI_n(\delta))\big)\leq  L\cdot \log \big(\pi_\beta(\cB^c_{\epsilon,r_\epsilon})+\delta_\beta\big)
 \leq L\cdot \max\big\{\log \big(2\pi_\beta(\cB^c_{\epsilon,r_\epsilon}), \log(2\delta_\beta)\big\}.
 \end{align*}
Finally, in above proof, by replacing $\epsilon$ with $\epsilon/3$, we obtain the result.
\end{proof}

\end{document}